\documentclass{aamas2016}
\usepackage{subfig}
\usepackage{wrapfig}
\usepackage{times}
\usepackage{url}
\usepackage{bm}
\usepackage{graphicx}
\usepackage{amsmath,amssymb}
\usepackage{algorithmic}
\usepackage[linesnumbered,ruled,vlined]{algorithm2e}
\usepackage{diagbox}
\usepackage{makecell}
\usepackage{floatrow}

\newtheorem{problem}{Problem}
\newtheorem{theorem}{Theorem}[section]

\DeclareMathOperator*{\argmax}{argmax}
\DeclareMathOperator{\E}{\mathbb{E}}

\newcommand{\tuple}[1]{\ensuremath{\left \langle #1 \right \rangle }}

\begin{document}
\sloppy
\title{Using Social Networks to Aid Homeless Shelters:\\
Dynamic Influence Maximization under Uncertainty - An Extended Version
}

\author{Amulya Yadav, Hau Chan$^1$, Albert Jiang$^1$, Haifeng Xu, Eric Rice, Milind Tambe\\
University of Southern California, Los Angeles, CA 90089\\
$^1$Trinity University, San Antonio, TX 78212\\
\{amulyaya, haifengx, ericr, tambe\}@usc.edu  $^1$\{hchan, xjiang\}@trinity.edu
}

\maketitle
\begin{abstract}
This paper presents HEALER, a software agent that recommends sequential intervention plans for use by homeless shelters, who organize these interventions to raise awareness about HIV among homeless youth. HEALER's sequential plans (built using knowledge of social networks of homeless youth) choose intervention participants strategically to maximize influence spread, while reasoning about uncertainties in the network. While previous work presents influence maximizing techniques to choose intervention participants, they do not address three real-world issues: (i) they \textit{completely fail} to scale up to real-world sizes; (ii) they do not handle deviations in execution of intervention plans; (iii) constructing real-world social networks is an expensive process. HEALER handles these issues via four major contributions: (i) HEALER casts this influence maximization problem as a POMDP and solves it using a novel planner which scales up to previously unsolvable real-world sizes; (ii) HEALER allows shelter officials to modify its recommendations, and updates its future plans in a deviation-tolerant manner; (iii) HEALER constructs social networks of homeless youth at low cost, using a Facebook application. Finally, (iv) we show hardness results for the problem that HEALER solves. HEALER will be deployed in the real world in early Spring 2016 and is currently undergoing testing at a homeless shelter.
\end{abstract}




\section{Introduction}\label{sec:intro}
HIV-AIDS kills 2 million people worldwide every year \cite{unaids}. In USA alone, AIDS kills around 10,000 people per annum \cite{cdc}. HIV has an extremely high incidence among homeless youth, as they are more likely to engage in high HIV-risk behaviors (e.g., unprotected sexual activity, injection drug use) than other sub-populations. In fact, previous studies show that homeless youth are at 10X greater risk of HIV infection than stably housed populations \cite{nchc}. Thus, any attempt at eradicating HIV crucially depends on our success at minimizing rates of HIV infection among homeless youth.

As a result, many homeless shelters organize intervention camps for homeless youth in order to raise awareness about HIV prevention and treatment practices. These intervention camps consist of day-long educational sessions in which the participants are provided with information about HIV prevention measures \cite{rice2012mobilizing}.

However, due to financial/manpower constraints, the shelters can only organize a limited number of intervention camps. Moreover, in each camp, the shelters can only manage small groups of youth ($\sim$3-4) at a time (as emotional and behavioral problems of youth makes management of bigger groups difficult). Thus, the shelters prefer a series of small sized camps organized sequentially \cite{rice2012}. As a result, the shelter cannot intervene on the entire target (homeless youth) population. Instead, it tries to maximize the spread of awareness among the target population (via word-of-mouth influence) using the limited resources at its disposal. To achieve this goal, the shelter uses the friendship based social network of the target population to strategically choose the participants of their limited intervention camps. Unfortunately, the shelters' job is further complicated by a lack of complete knowledge about the social network's structure \cite{rice2010positive}. Some friendships in the network are known with certainty whereas there is uncertainty about other friendships. 

Thus, the shelters face an important challenge: they need a sequential plan to choose the participants of their sequentially organized interventions. This plan must address four key points: (i) it must deal with network structure uncertainty;  (ii) it needs to take into account new information uncovered during the interventions, which reduces the uncertainty in our understanding of the network; (iii) the plan needs to be deviation tolerant, as sometimes homeless youth may refuse to be an intervention participant, thereby forcing the shelter to modify its plan; (iv) the intervention approach should address the challenge of gathering information about social networks of homeless youth, which usually costs thousands of dollars and many months of time \cite{rice2012}. 


In this paper, we model the shelters' problem by introducing the \textit{\textbf{D}ynamic \textbf{I}nfluence \textbf{M}aximization under Unc\textbf{e}rtainty} (or DIME) problem. The sequential selection of intervention participants under network uncertainty in DIME sets it apart from any other previous work on influence maximization, which mostly focuses on single shot choices \cite{Borgs14,tang2014influence,kempe2003maximizing,leskovec2007cost}. Additionally, in previous work, PSINET \cite{yadav2015preventing}, a POMDP based tool, was proposed for solving this problem, but it has three limitations. First, PSINET completely fails to scale up to the problem's requirements; running slowly and out of memory. It runs very slowly for moderate-sized networks, and runs out of memory as the network is scaled up. Worse still, even on these moderate sized networks, it runs out of memory when the number of participants in an intervention are increased (as shown later). Second, PSINET did not explicitly allow officials to modify its recommended plans if some participants refuse to attend the intervention. Third, PSINET requires entire social networks of homeless youth as input, while homeless shelters lack the money/time/manpower required to generate these input networks.  

In this paper, we build a new software agent, HEALER (\textbf{H}ierarchical \textbf{E}nsembling based \textbf{A}gent which p\textbf{L}ans for \textbf{E}ffective \textbf{R}eduction in HIV Spread), to provide an end-to-end solution to the DIME problem. HEALER addresses PSINET's shortcomings via four contributions. First, HEALER casts the DIME problem as a Partially Observable Markov Decision Process (POMDP) and solves it using HEAL (\textbf{H}ierarchical \textbf{E}nsembling \textbf{A}lgorithm for p\textbf{L}anning), a novel POMDP planner which quickly generates high-quality recommendations (of intervention participants) for homeless shelter officials. HEAL uses a hierarchical ensembling heuristic to ensure low memory utilization, thereby enabling scale up. HEAL hierarchically subdivides our \textit{original POMDP} at two layers: (i) In the top layer, graph partitioning techniques are used to divide the \textit{original POMDP} into \textit{intermediate POMDPs}; (ii) In the second level, each of these \textit{intermediate POMDPs} is further simplified by sampling uncertainties in network structure repeatedly to get \textit{sampled POMDPs}; (iii) Finally, we use aggregation techniques to combine the solutions to these simpler POMDPs, in order to generate the overall solution for the \textit{original POMDP}. Our simulations show that even on small settings, HEAL achieves a 100X speed up and 70\% improvement in solution quality over PSINET; and on larger problems \textit{where PSINET is unable to run at all}, HEAL continues to provide high quality solutions quickly. Second, HEALER tolerates deviations in execution of intervention plans, as it periodically receives feedback from shelter officials about executed plans, reasons about any deviations from its recommended plans, and updates its plan accordingly to maximize solution quality. Third, HEALER quickly gathers information about the homeless youth social network (at low cost) by interacting with youth via a Facebook application. Fourth, we analyze several novel theoretical aspects of the DIME problem, which illustrates its hardness.

\begin{figure}[htp]
\subfloat[\small Computers at Homeless Shelter where HEALER is deployed]{\includegraphics[height=0.9in,width=0.48\columnwidth]{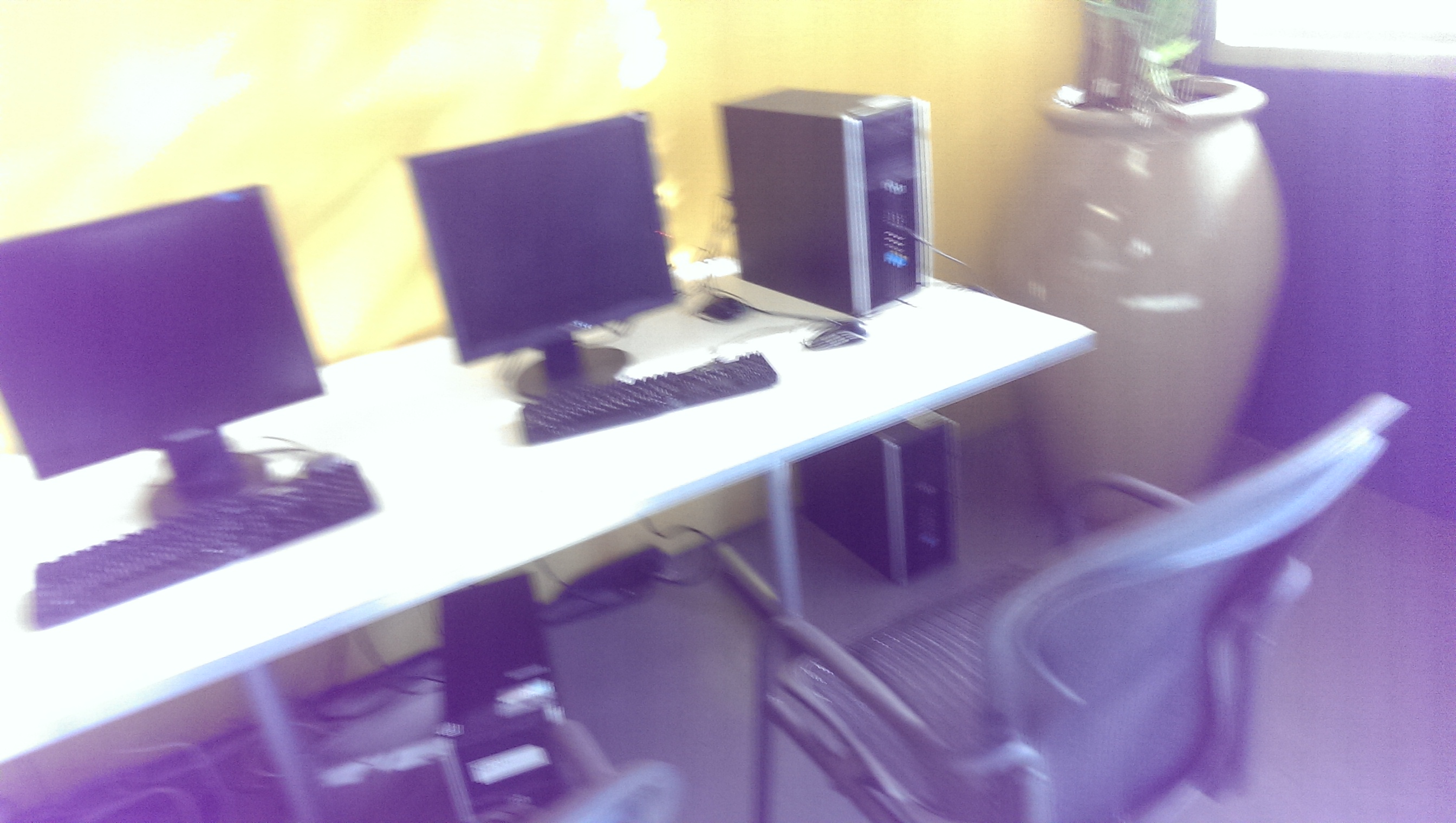}\label{fig:shelter1}}
\hspace{2mm}
\subfloat[\small Emergency Resource Shelf at the Homeless Shelter]{\includegraphics[height=0.9in,width=0.48\columnwidth]{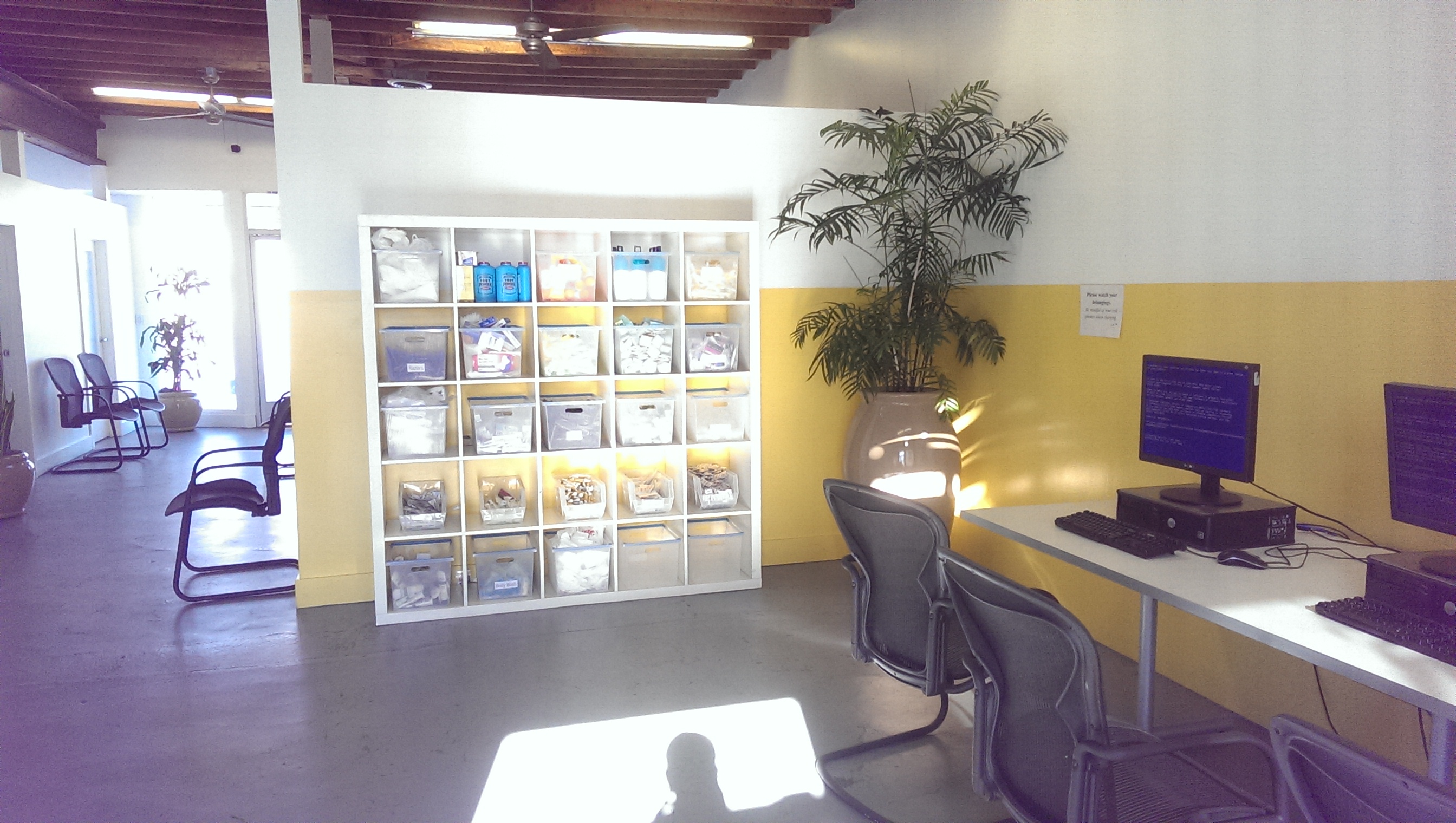}\label{fig:shelter2}}
\caption{\small Facilities at our Collaborating Homeless Shelter}
\end{figure}

We deploy HEALER in a real-world pilot study, in collaboration with a homeless shelter (name withheld for anonymity), which provides food and lodging to homeless youth aged 12-25. They provide these facilities for $\sim$55-60 homeless youth every day. They also operate an on-site medical clinic where free HIV and Hepatitis-C testing is provided. HEALER has been reviewed by officials at our collaborating homeless shelter and their feedback has been positive. We are currently preparing to register 100 youth in our deployment of HEALER at this shelter. To the best of our knowledge, this pilot study represents the first real-world evaluation of such sequential influence maximization algorithms. We expect deployment to commence in early Spring 2016.

\section{Related work}
First, we discuss work related to influence maximization. There are many algorithms for finding `seed sets' of nodes to maximize influence spread in networks \cite{kempe2003maximizing,leskovec2007cost,Borgs14,tang2014influence}. However, all these algorithms assume \textit{no uncertainty in the network structure} and select a single seed set. In contrast, we select several seed sets sequentially in our work to select intervention participants. Also, our problem takes into account uncertainty about the network structure and influence status of network nodes (i.e., whether a node is influenced or not). Finally, unlike \cite{kempe2003maximizing,leskovec2007cost,Borgs14,tang2014influence}, we use a different diffusion model as we explain later. Golovin et. al. \cite{golovin2011adaptive} introduced adaptive submodularity and discussed adaptive sequential selection (similar to our problem), and they proved that a Greedy algorithm has a $(1-1/e)$ approximation guarantee. However, unlike our work, they assume no uncertainty in network structure. Also, while our problem can be cast into the adaptive stochastic optimization framework of \cite{golovin2011adaptive}, our influence function is not adaptive submodular (see Section \ref{sec:DIME}), because of which their Greedy algorithm loses its approximation guarantees. 

Next, we discuss literature from \textit{social work}. The general approach to these interventions is to use  Peer Change Agents (PCA) (i.e., peers who bring about change in attitudes) to engage homeless youth in interventions, but most studies don't use network characteristics to choose these PCAs \cite{schneider2015new}. A notable exception is Valente et. al. \cite{valente2007identifying}, who proposed selecting intervention participants with highest \textit{degree centrality} (the most ties to other homeless youth). However, previous studies \cite{cohen2014sketch,yadav2015preventing} show that \textit{degree centrality} performs poorly, as it does not account for potential overlaps in influence of two high degree centrality nodes.


The final field of related work is planning for reward/cost optimization. We only focus on the literature on Monte-Carlo (MC) sampling based online POMDP solvers since this approach allows significant scale-up \cite{ross2008online}. The POMCP solver \cite{silver2010monte}  uses Monte-Carlo UCT tree search in online POMDP planning. Also, Somani et. al. \cite{somani2013despot} present the DESPOT algorithm, that improves the worst case performance of POMCP. Our initial experiments with POMCP and DESPOT showed that they run out of memory on even our small sized networks. A recent paper \cite{yadav2015preventing} introduced PSINET-W, a MC sampling based online POMDP planner. We have discussed PSINET's shortcomings in Section \ref{sec:intro} and how we remedy them. In particular, as we show later, HEALER scales up whereas PSINET fails to do so. \textit{HEALER's algorithmic approach also offers significant novelties in comparison with PSINET (see Section \ref{sec:alg})}. Further, a recent paper \cite{leo2016aamas} looks at an extension of the same problem by considering the case that not all nodes in the network are known ahead of time (as opposed to our work where we only assume that some edges are not known ahead of time). However, unlike our work, they do not consider sequential selection of node subsets.


\section{HEALER's Design}
We now explain the high-level design of HEALER. It consists of two major components: (i) a Facebook application for gathering information about social networks; and (ii) a DIME Solver, which solves the DIME problem (introduced in Section \ref{sec:DIME}). We first explain HEALER's components and then explain HEALER's design.

\textbf{Facebook Application: }HEALER gathers information about social ties in the homeless youth social network by interacting with youth via a Facebook application. We choose Facebook for gathering information as Young et. al. \cite{young2011online} show that $\sim$80\% of homeless youth are active on Facebook. Once a fixed number of homeless youth register in the Facebook application, HEALER parses the Facebook contact lists of all the registered homeless youth and generates the social network between these youth. HEALER adds a link between two people, if and only if both people are (i) friends on Facebook; and (ii) are registered in its Facebook application. Unfortunately, there is \textit{uncertainty} in the generated network as friendship links between people who are only friends in real-life (and not on Facebook) are not captured by HEALER. 

Previously, homeless shelters gathered this social network information via tedious face-to-face interviews with homeless youth, a process which cost thousands of dollars and many months of time. HEALER's Facebook application allows homeless shelters to quickly generate a (partial) homeless youth social network at low cost. This Facebook application has been tested by our collaborating homeless shelter with positive feedback.

\textbf{DIME Solver: }The DIME Solver then takes the approximate social network (generated by HEALER's Facebook application) as input and solves the DIME problem (formally defined in Section \ref{sec:DIME}) using our new algorithm (explained in Section \ref{sec:alg}). HEALER provides the solution of this DIME problem as a series of recommendations (of intervention participants) to homeless shelter officials. 

\textbf{HEALER Design: }HEALER's design (shown in Figure \ref{fig:HEALER}), begins with the Facebook application constructing an \textit{uncertain} network (as explained above). HEALER has a \textit{sense-reason-act} cycle;  where it repeats the following process for $T$ interventions.

\begin{figure}[t]
\center{\includegraphics[width=70mm]
{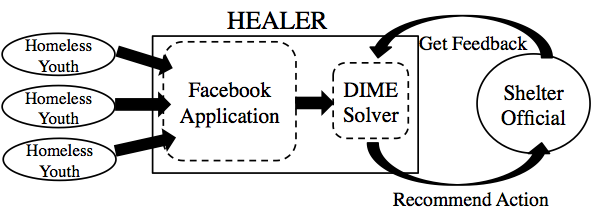}}
\caption{\label{fig:HEALER} HEALER's Design}
\end{figure}

It \textit{reasons} about different long-term plans to solve the DIME problem, it \textit{acts} by providing DIME's solution as a recommendation (of intervention participants) to homeless shelter officials. The officials may choose to not use HEALER's recommendation in selecting their intervention's participants. Upon the intervention's completion, HEALER \textit{senses} feedback about the conducted intervention from the officials. This feedback includes new observations about the network, e.g., uncertainties in some links may be resolved as intervention participants are interviewed by the shelter officials (explained more in Section \ref{sec:DIME}). HEALER uses this feedback to update and improve its future recommendations.

\section{Network Generation}\label{sec:deploy}
First, we explain our model for influence spread in \textit{uncertain social networks}. Then, we describe how HEALER generates a social network using its' Facebook application.

\subsection{Background}\label{sec:prelims}
We represent social networks as directed graphs (consisting of \textit{nodes} and \textit{directed edges}) where each \textit{node} represents a person in the social network and a \textit{directed edge} between two nodes $A$ and $B$ (say) represents that node $A$ \textit{considers} node $B$ as his/her friend. \textit{We assume directed-ness of edges as sometimes homeless shelters assess that the influence in a friendship is very much uni-directional; and to account for uni-directional follower links on Facebook}. Otherwise friendships are encoded as two uni-directional links.

 
\textbf{Uncertain Network}: The uncertain network is a directed graph $G=(V,E)$  with $|V| = N$ nodes and $|E| = M$ edges. The edge set $E$ consists of two disjoint subsets of edges: $E_c$ (the set of certain edges, i.e., friendships which we are certain about) and $E_u$ (the set of uncertain edges, i.e., friendships which we are uncertain about). Note that uncertainties about friendships exist because HEALER's Facebook application misses out on some links between people who are friends in real life, but not on Facebook.


To model the uncertainty about missing edges, every uncertain edge $e \in E_u$ has an existence probability $u(e)$ associated with it, which represents the likelihood of ``existence" of that uncertain edge. For example, if there is an uncertain edge $(A,B)$ (i.e., we are unsure whether node $B$ is node $A$'s friend), then $u(A,B) = 0.75$ implies that $B$ is $A$'s friend with a 0.75 chance. In addition, each edge $e \in E$ (both certain and uncertain) has a propagation probability $p(e)$ associated with it. A propagation probability of 0.5 on directed edge $(A,B)$ denotes that if node $A$ is influenced (i.e., has information about HIV prevention), it influences node $B$ (i.e., gives information to node $B$) with a 0.5 probability in each subsequent time step (our full influence model is defined below). This graph $G$ with all relevant $p(e)$ and $u(e)$ values represents an uncertain network and serves as an input to the DIME problem. Figure \ref{fig:uncertainG} shows an uncertain network on 6 nodes (\textit{A} to \textit{F}) and 7 edges. The dashed and solid edges represent uncertain (edge numbers 1, 4, 5 and 7) and certain (edge numbers 2, 3 and 6) edges, respectively. Next, we explain the influence diffusion model that we use in HEALER.


\begin{wrapfigure}{r}{0.15\textwidth}
\centering
\includegraphics[width=1in]{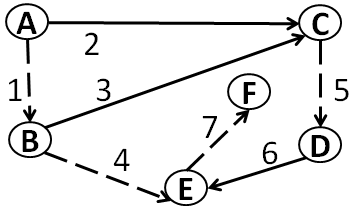}
\caption{\label{fig:uncertainG} Uncertain Network}
\end{wrapfigure}


\textbf{Influence Model} We use a variant of the independent cascade model \cite{yan2011influence}. In the standard independent cascade model, all nodes that get influenced at round $t$ get a \textbf{single} chance to influence their un-influenced neighbors at time $t+1$. If they fail to spread influence in this \textbf{single} chance, they don't spread influence to their neighbors in future rounds. Our model is different in that we assume that nodes get \textbf{multiple} chances to influence their un-influenced neighbors. If they succeed in influencing a neighbor at a given time step $t'$, they stop influencing that neighbor for all future time steps. Otherwise, if they fail in step $t'$, they try to influence again in the next round. This variant of independent cascade has been shown to empirically provide a better approximation to real influence spread than the standard independent cascade model \cite{cointet2007, yan2011influence}. Further, we assume that nodes that get influenced at a certain time step remain influenced for all future time steps. We now explain how HEALER generates an \textit{uncertain social network}.



\subsection{HEALER's Facebook application}
HEALER generates an \textit{uncertain network} by (i) using its Facebook application to generate a network with no uncertain edges; (ii) using well known link prediction techniques such as KronEM \cite{kim2011network} to infer existence probabilities $u(e)$ for all possible \textit{missing} edges that are not present in the network; (iii) deciding on a threshold probability $\tau$ (in consultation with homeless shelter officials), so that we \textit{only} add a \textit{missing} edge as an uncertain edge if its inferred existence probability $u(e) > \tau$; and (iv) asking homeless shelter officials to provide $p(e)$ estimates for network edges. 

\textbf{Choosing $\tau$: } Rice et. al \cite{rice2012position} show that real-world homeless youth networks are \textit{relatively sparse}. Thus, shelter officials choose the threshold probability value $\tau$ such that the number of uncertain edges that get added because of $\tau$ does not make our input uncertain network \textit{overly dense}. Next, we introduce the DIME problem.


\section{DIME Problem}\label{sec:DIME}
We now provide some background information that helps us define a precise problem statement for DIME. After that, we will show some hardness results about this problem statement. 

Given the \textit{uncertain network} as input, HEALER runs for \textit{$T$ rounds} (corresponding to the number of interventions organized by the homeless shelter). 
In each round, HEALER chooses \textit{$K$ nodes} (youth) as intervention participants. These participants are assumed to be influenced post-intervention with certainty. Upon influencing the chosen nodes, HEALER `\textit{observes}' the true state of the \textit{uncertain edges} (friendships) out-going from the selected nodes. This translates to asking intervention participants about their 1-hop social circles, which is within the homeless shelter's capabilities \cite{rice2012position}. 

After each round, influence spreads in the network according to our influence model for \textit{$L$ time steps}, before we begin the next round. This $L$ represents the time duration in between two successive intervention camps. \textit{In between rounds, HEALER does not observe the nodes that get influenced during $L$ time steps}. HEALER only knows that explicitly chosen nodes (our intervention participants in all past rounds) are influenced. Informally then, given an uncertain network $G_0=(V, E)$ and integers $T$, $K$, and $L$ (as defined above), HEALER finds an online policy for choosing \textit{exactly} $K$ nodes for $T$ successive rounds (interventions) which maximizes influence spread in the network at the end of $T$ rounds. 

We now provide notation for defining HEALER's policy formally. Let $\bm{\mathcal{A}}=\{A \subset V \mbox{ s.t. } |A|=K \}$ denote the set of $K$ sized subsets of $V$, which represents the set of possible choices that HEALER can make at every time step $t \in [1,T]$. Let $A_i \in \bm{\mathcal{A}} \mbox{ } \forall i \in [1,T]$ denote HEALER's choice in the $i^{th}$ time step. Upon making choice $A_i$, HEALER `\textit{observes}' uncertain edges adjacent to nodes in $A_i$, which updates its understanding of the network. Let $G_i \mbox{ } \forall \mbox{ } i \in [1,T]$ denote the uncertain network resulting from $G_{i-1}$ with \textit{observed} (additional edge) information from $A_i$. Formally, we define a history $H_i \mbox{ } \forall \mbox{ } i \in [1,T]$ of length $i$ as a tuple of past choices and observations $H_i = \tuple{G_0, A_1, G_1, A_2,..,A_{i-1},G_i}$. Denote by $\bm{\mathcal{H}_i} = \{ H_k \mbox{ s.t. } k \leqslant i \}$ the set of all possible histories of length less than or equal to $i$. Finally, we define an $i$-step policy $\bm{\Pi_i} \colon \bm{\mathcal{H}_i} \to \bm{\mathcal{A}}$ as a function that takes in histories of length less than or equal to $i$ and outputs a $K$ node choice for the current time step. We now provide an explicit problem statement for DIME.

\begin{problem}{\textbf{DIME Problem}}
Given as input an uncertain network $G_0=(V, E)$ and integers $T$, $K$, and $L$ (as defined above). Denote by $\mathcal{R}(H_T, A_T)$ the \textit{expected total number of influenced nodes at the end of round $T$}, given the $T$-length history of previous observations and actions $H_T$, along with $A_T$, the action chosen at time $T$. Let $\E_{H_T,A_T \sim \Pi_T} [\mathcal{R}(H_T,A_T)]$ denote the expectation over the random variables $H_T=\tuple{G_0, A_1,..,A_{T-1},G_T}$ and $A_T$, where $A_i$ are chosen according to $\Pi_T(H_i)  \mbox{ } \forall \mbox{ } i \in [1,T]$, and $G_i$ are drawn according to the distribution over uncertain edges of $G_{i-1}$ that are revealed by $A_i$. The objective of DIME is to find an optimal $T$-step policy $\bm{\Pi_T^*} = \argmax_{\Pi_T} \E_{H_T,A_T \sim \Pi_T}[\mathcal{R}(H_T, A_T)]$. 
\end{problem} 

Next, we show hardness results about the DIME problem. First, we analyze the value of having complete information in DIME. Then, we characterize the computational hardness of DIME.

\textbf{The Value of Information. }We characterize the impact of insufficient information (about the uncertain edges) on the achieved solution value. We show that no algorithm for DIME is able to provide a good approximation to the \textit{full-information solution value} (i.e., the best solution achieved w.r.t. the underlying ground-truth network), even with infinite computational power. 

\begin{figure}[htb]
\center{\includegraphics[scale=.33]
{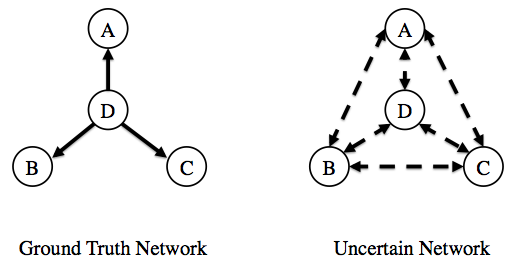}}
\caption{\label{fig:Figure1} Counter-example for Theorem \ref{Th:4}}
\end{figure}

\begin{theorem}\label{Th:4}
Given an uncertain network with $n$ nodes, for any $\epsilon > 0$, there is no algorithm for the DIME problem which can guarantee a $n^{-1+\epsilon}$ approximation to $OPT_{full}$, the \textit{full-information solution value}. 
\end{theorem}
\begin{proof}
We prove this statement by providing a counter-example in the form of a specific (ground truth) network for which there can exist no algorithm which can guarantee a $n^{-1+\epsilon}$ approximation to $OPT_{full}$. Consider an input to the DIME problem, an \textit{uncertain network} with $n$ nodes with $2 * {n \choose 2}$ uncertain edges between the $n$ nodes, i.e., it's a completely connected uncertain network consisting of \textit{only} uncertain edges (an example with $n=3$ is shown in Figure \ref{fig:Figure1}). Let $p(e)=1$ and $u(e)=0.5$ on all edges in the \textit{uncertain network}, i.e., all edges have the same propagation and existence probability. Let $K=1$, $L=1$ and $T=1$, i.e., we just select a single node in one shot (in a single round). 

Further, consider a star graph (as the ground truth network) with $n$ nodes such that propagation probability $p(e) = 1$ on all edges of the star graph (shown in Figure 1). Now, any algorithm for the DIME problem would select a single node in the \textit{uncertain network} uniformly at random with equal probability of $1/n$ (as information about all nodes is symmetrical). In expectation, the algorithm will achieve an expected reward  $\{1/n \times (n)\} + \{1/n \times (1) + ... + 1/n \times (1)\} = 1/n \times(n) + (n-1)/n \times 1 = 2 - 1/n$. However, given the ground truth network, we get $OPT_{full}=n$, because we always select the star node. As $n$ goes to infinity, we can at best achieve a $n^{-1}$ approximation to $OPT_{full}$. Thus, no algorithm can achieve a $n^{-1+\epsilon}$ approximation to $OPT_{full}$ for any $\epsilon > 0$.   
\end{proof}

\textbf{Computational Hardness. }We now analyze the hardness of computation in the DIME problem in the next two theorems.

\begin{theorem}\label{Th:1}
The DIME problem is NP-Hard.
\end{theorem}
\begin{proof}
Consider the case where $E_u = \Phi$, $L=1$, $T=1$ and $p(e) = 1\mbox{ } \forall \mbox{ } e \in E$. This degenerates to the standard influence maximization problem which is shown to be NP-Hard \cite{kempe2003maximizing}. Thus, the DIME problem is also NP-Hard.   
\end{proof}

Some NP-Hard problems exhibit nice properties that enable approximation guarantees for them. Golovin et. al. \cite{golovin2011adaptive} introduced adaptive submodularity, an analog of submodularity for adaptive settings. Presence of adaptive submodularity ensures that a simply greedy algorithm provides a $(1-1/e)$ approximation guarantee w.r.t. the optimal solution defined on the \textit{uncertain network}. However, as we show next, while DIME can be cast into the adaptive stochastic optimization framework of \cite{golovin2011adaptive}, our influence function is not adaptive submodular, because of which their Greedy algorithm does not have a $(1-1/e)$ approximation guarantee. 

\begin{figure}
		\begin{minipage}[t]{0.8\columnwidth}%
			\includegraphics[bb=120bp 200bp 770bp 400bp,clip,scale=0.5]{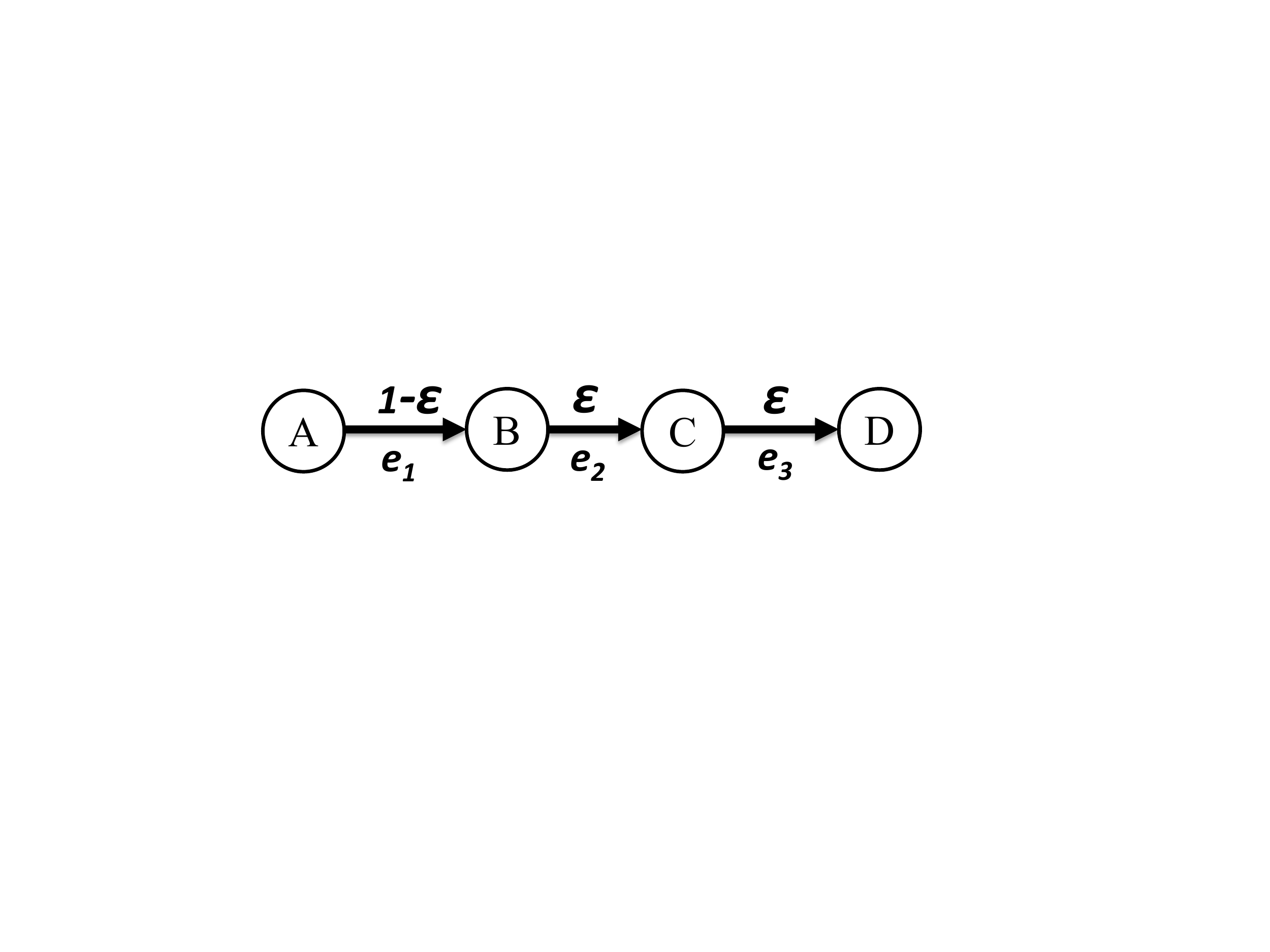}%
		\end{minipage}
		\caption{\label{fig:failure}Failure of Adaptive Submodularity}
	\end{figure}

\begin{theorem}\label{Th:2}
The influence function of DIME is not adaptive submodular.
\end{theorem}
\begin{proof}
	The definition of adaptive submodularity requires that the expected marginal increase of influence by picking an additional node $v$ is more when we have less observation. Here the expectation is taken over the random states that are consistent with current observation. We show that this is not the case in DIME problem. Consider a path with $4$ nodes $a,b,c,d$ and three \emph{directed} edges $e_1 = (a,b)$ and $e_2 = (b,c)$ and $e_3 = (c,d)$ (see Figure \ref{fig:failure}). Let $p(e_1) = p(e_2) = p(e_3)=1$, i.e., propagation probability is $1$; $L=2$, i.e., influence stops after two round; and $u(e_1) =1-\epsilon$ $u(e_2) = u(e_3) = \epsilon$ for some small enough $\epsilon$ to be set. That is the only uncertainty comes from incomplete knowledge of the existence of edges. 
	
	Let $\Psi_1 = \{e_1 \text{ exists} \}$ and $\Psi_2 = \{e_1,e_3 \text{ exists} \}$. Then $\mathbb{E}_{\Phi}\left[f(a,b,c)|\Phi \sim \Psi_2 \right] = 4$ since all nodes will be influenced. $\mathbb{E}_{\Phi}\left[f(a,c)|\Phi \sim \Psi_2 \right] = 4-\epsilon$ since the only uncertain node is $b$ which will be influenced with probability $1-\epsilon$. Therefore,
	\begin{equation}\label{eq:withb}
		\mathbb{E}_{\Phi}\left[f(a,b,c)|\Phi \sim \Psi_2 \right] - \mathbb{E}_{\Phi}\left[f(a,c)|\Phi \sim \Psi_2 \right] = \epsilon.
	\end{equation}
		Now 
	 $\mathbb{E}_{\Phi}\left[f(a,b)|\Phi \sim \Psi_1 \right] = 2 + \epsilon + \epsilon^2$ since $a,b$ will be surely influenced, $c$ and $d$ will be influenced with probability $\epsilon$ and $\epsilon^2$ respectively. On the other hand, $\mathbb{E}_{\Phi}\left[f(a)|\Phi \sim \Psi_1 \right] = 2 + \epsilon $ since $b$ will be surely influenced (since $e_1$ exists) and $c$ will be influenced with probability $\epsilon$. Since $L=2$, $d$ cannot be influenced. As a result, 
	 	\begin{equation}\label{eq:nob}
	 	\mathbb{E}_{\Phi}\left[f(a,b)|\Phi \sim \Psi_2 \right] - \mathbb{E}_{\Phi}\left[f(a)|\Phi \sim \Psi_2 \right] = \epsilon^2.
	 	\end{equation}
	 	
	 	Combining Equation \eqref{eq:withb} and \eqref{eq:nob}, we know that DIME is not adaptive submodular.
\end{proof}

\section{HEAL: DIME PROBLEM SOLVER}\label{sec:pomdp}
The above theorems show that DIME is a hard problem as it is difficult to even obtain any reasonable approximations. We model DIME as a POMDP \cite{puterman2009markov} because of two reasons. First, POMDPs are a good fit for DIME as (i) we conduct several interventions sequentially, similar to sequential POMDP actions; and (ii) we have \textit{partial observability} (similar to POMDPs) due to uncertainties in network structure and influence status of nodes. Second, POMDP solvers have recently shown great promise in generating near-optimal policies efficiently \cite{silver2010monte}. We now 
explain how we map DIME onto a POMDP. 


\textbf{States. } A POMDP state in our problem is a pair of binary tuples $s = \tuple{W, F}$ where $W$ and $F$ are of lengths $|V|$ and $|E_U|$, respectively. Intuitively, $W$ denotes the influence status of network nodes, where $W_i = 1$ denotes that node $i$ is influenced and $W_i = 0$ otherwise. Moreover, $F$ denotes the existence of uncertain edges, where $F_i = 1$ denotes that the $i^{th}$ uncertain edge exists in reality, and $F_i = 0$ otherwise.



\textbf{Actions. } Every choice of a subset of $K$ nodes is a POMDP action. More formally, $A = \{ a \subset V s.t. |a| = K\}$. For example, in Figure \ref{fig:uncertainG}, one possible action is $\{A,B\}$ (when $K=2$). 

\textbf{Observations. }Upon taking a POMDP action, we ``\textit{observe}" the ground reality of the uncertain edges outgoing from the nodes chosen in that action. Consider $\Theta(a) = \{ \mbox{e }|\mbox{ e = (x,y) \text{s.t.} x} \in a \mbox{ } \wedge\mbox{ e} \in E_u \}\mbox{ }\forall a \in A$, which represents the (ordered) set of uncertain edges that are observed when we take action $a$. Then, our POMDP observation upon taking action $a$ is defined as $o(a) = \{F_{e} | e \in \Theta(a)\}$, i.e., the F-values of the observed uncertain edges. For example, by taking action $\{B,C\}$ in Figure \ref{fig:uncertainG}, the values of $F_4$ and $F_5$ (i.e., the F-values of uncertain edges in the 1-hop social circle of nodes $B$ and $C$) would be observed.

\textbf{Rewards. } The reward $R(s,a,s')$ of taking action $a$ in state $s$ and reaching state $s'$ is the number of newly influenced nodes in $s'$. More formally, $R(s,a,s') = (\|s'\|-\|s\|)$, where $\|s'\|$ is the number of influenced nodes in $s'$. 

\textbf{Initial Belief State. } The initial belief state is a distribution $\beta_0$ over all states $s \in S$. The support of $\beta_0$ consists of all states $s = \tuple{W, F}$ s.t. $W_i=0 \mbox{ } \forall \mbox{ } i \in [1,|V|]$, i.e., all states in which all network nodes are un-influenced (as we assume that all nodes are un-influenced to begin with). Inside its support, each $F_i$ is distributed independently according to $P(F_i=1)= u(e)$.

\textbf{Transition And Observation Probabilities. } Computation of exact transition probabilities $T(s'|s,a)$ requires considering all possible paths in a graph through which influence could spread, which is $\mathcal{O}(N!)$ ($N$ is number of nodes in the network) in the worst case. Moreover, for large social networks, the size of the transition and observation probability matrix is prohibitively large (due to exponential sizes of state and action space).

Therefore, instead of storing huge transition/observation matrices in memory, we follow the paradigm of large-scale online POMDP solvers \cite{silver2010monte,eck2015ask,dibangoye2009topological} by using a generative model $\Lambda(s, a) \sim (s', o, r)$ of the transition and observation probabilities. This generative model allows us to generate on-the-fly samples from the exact distributions $T(s'|s,a)$ and $\Omega(o|a,s')$ at very low computational costs. Given an initial state $s$ and an action $a$ to be taken, our generative model $\Lambda$ simulates the random process of influence spread to generate a random new state $s'$, an observation $o$ and the obtained reward $r$. Simulation of the random process of influence spread is done by ``\textit{playing}" out propagation probabilities (i.e., flipping weighted coins with probability $p(e)$) according to our influence model to generate sample $s'$. The observation sample $o$ is then determined from $s'$ and $a$. Finally, the reward sample $r = (\|s'\|-\|s\|)$ (as defined above). This simple design of the generative model allows significant scale and speed up (as seen in previous work \cite{silver2010monte} and also in our experimental results section).

We solve this POMDP using a novel algorithm (described in Section \ref{sec:alg}) to find the optimal policy $\bm{\Pi_T^*}$ for the DIME problem.



\subsection{HEALER's DIME Solver}\label{sec:alg}
Initial experiments with the POMCP solver \cite{silver2010monte} showed that it ran out of memory on 30 node graphs. Similarly, PSINET-W \cite{yadav2015preventing} was simply unable to scale up to real world demands (as shown in our experiments). Hence, we propose HEAL, a new heuristic based online POMDP planner (for solving the DIME problem) which scales up to our collaborating shelter's real world demands.

\subsubsection{HEAL}
HEAL solves the \textit{original POMDP} using a novel \textit{hierarchical ensembling heuristic}: it creates ensembles of imperfect (and smaller) POMDPs at \textit{two} different layers, in a hierarchical manner (see Figure \ref{fig:Flow}). HEAL's \textit{top layer} creates an ensemble of smaller sized \textit{intermediate POMDPs} by subdividing the original \textit{uncertain network} into several smaller sized \textit{partitioned networks} by using graph partitioning techniques \cite{lasalle2013multi}. Each of these partitioned networks is then mapped onto a POMDP, and these \textit{intermediate POMDPs} form our \textit{top layer} ensemble of POMDP solvers.

\begin{figure}[t]
\center{\includegraphics[scale=.35]
{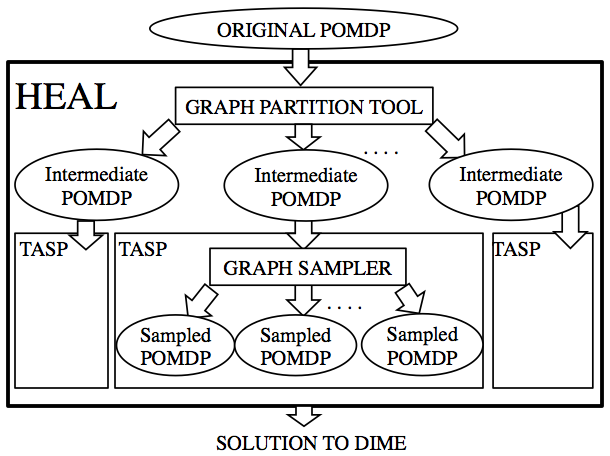}}
\caption{\label{fig:Flow} Hierarchical decomposition in HEAL}
\vspace{-2mm}
\end{figure} 

In the bottom layer, each \textit{intermediate POMDP} is solved using TASP (\textbf{T}ree \textbf{A}ggregation for \textbf{S}equential \textbf{P}lanning), our novel POMDP planner, which subdivides the POMDP into another ensemble of smaller sized \textit{sampled POMDPs}. Each member of this \textit{bottom layer} ensemble is created by randomly sampling uncertain edges of the partitioned network to get a sampled network having no uncertain edges, and this sampled network is then mapped onto a \textit{sampled POMDP}. Finally, the solutions of POMDPs in both the \textit{bottom} and \textit{top layer} ensembles are aggregated using novel techniques to get the solution for HEAL's original POMDP. 

HEAL uses several novel heuristics. First, it uses a novel two-layered \textit{hierarchical ensembling heuristic}. Second, it uses graph partitioning techniques to partition the uncertain network, which generates partitions that minimize the edges going across partitions (while ensuring that partitions have similar sizes). Since these partitions are ``almost" disconnected, we solve each partition separately. Third, it solves the \textit{intermediate POMDP} for each partition by creating smaller-sized \textit{sampled POMDPs} (via sampling uncertain edges), each of which is solved using a novel tree search algorithm, which avoids the exponential branching factor seen in PSINET \cite{yadav2015preventing}. Fourth, it uses novel aggregation techniques to combine solutions to these smaller POMDPs rather than simple plurality voting techniques seen in previous ensemble techniques \cite{yadav2015preventing}.

These heuristics enable scale up to real-world sizes (at the expense of sacrificing performance guarantees), as instead of solving one huge problem, we now solve several smaller problems. However, these heuristics perform very well in practice. Our simulations show that even on smaller settings, HEAL achieves a 100X speed up over PSINET, while providing a 70\% improvement in solution quality; and on larger problems, \textit{where PSINET is unable to run at all}, HEAL continues to provide high solution quality. Now, we elaborate on these heuristics by first explaining the TASP solver.

\begin{algorithm}[t!]
\label{alg:TASP}
\caption{TASP Solver}
\KwIn{Uncertain network $G$, Parameters $K$, $T$, $L$}
\KwOut{Best $K$ node action $\kappa$}
\begin{footnotesize}
Create ensemble of $\Delta$ different POMDPs; \\\label{flow:1}
\For {$\delta \in \Delta$} {
	$\alpha^{\delta} = Evaluate(\delta)$;\\\label{flow:2}
	}
	$r = Expectation(\alpha)$;\\\label{flow:3}
	$\kappa = \argmax_j r_j$;\\\label{flow:4}
	return $\kappa$;\\
\end{footnotesize}
\end{algorithm}
\setlength{\textfloatsep}{5pt}

\subsubsection{Bottom layer: TASP} \label{sec:TASP} We now explain TASP, our new POMDP solver that solves each \textit{intermediate POMDP} in HEAL's bottom layer. Given an \textit{intermediate POMDP} and the uncertain network it is defined on, as input, TASP goes through four steps (see Algorithm 1). 

First, Step \ref{flow:1} makes our \textit{intermediate POMDP} more tractable by creating an ensemble of smaller sized \textit{sampled POMDPs}. Each member of this ensemble is created by sampling uncertain edges of the input network to get an \textit{instantiated} network. Each uncertain edge in the input network is randomly kept with probability $u(e)$, or removed with probability $1-u(e)$, to get an \textit{instantiated} network with no uncertain edges. We repeat this sampling process to get $\Delta$ (a variable parameter) different \textit{instantiated} networks. These $\Delta$ different \textit{instantiated} networks are then mapped onto to $\Delta$ different POMDPs, which form our ensemble of \textit{sampled POMDPs}. Each \textit{sampled POMDP} shares the same action space (defined on the input partitioned network) as the different POMDPs only differ in the sampling of uncertain edges. Note that each member of our ensemble is a POMDP as even though sampling uncertain edges removes uncertainty in the $F$ portion of POMDP states, there is still partial observability in the $W$ portion of POMDP state. 

In Step \ref{flow:2} (called the Evaluate Step), for each instantiated network $\delta \in [1,\Delta]$, we generate an $\alpha^{\delta}$ list of rewards. The $i^{th}$ element of $\alpha^{\delta}$ gives the long term reward achieved by taking the $i^{th}$ action in \textit{instantiated} network $\delta$. In Step \ref{flow:3}, we find the expected reward $r_i$ of taking the $i^{th}$ action, by taking a reward expectation across the $\alpha^{\delta}$ lists (for each $\delta \in [1,\Delta]$) generated in the previous step. For e.g., if $\alpha^{\delta_1}_1 = 10$ and $\alpha^{\delta_2}_1 = 20$, i.e., the rewards of taking the $1^{st}$ action in instantiated networks $\delta_1$ and $\delta_2$ (which occurs with probabilities $P(\delta_1)$ and $P(\delta_2)$) are 10 and 20 respectively, then the expected reward $r_1 = P(\delta_1)\times 10 + P(\delta_2)\times 20$. Note that $P(\delta_1)$ and $P(\delta_2)$ are found by multiplying existence probabilities $u(e)$ (or $1-u(e)$) for uncertain edges that were kept (or removed) in $\delta_1$ and $\delta_2$. Finally, in Step \ref{flow:4}, the action $\kappa = \argmax_j r_j$ is returned by TASP. Next, we discuss the Evaluate Step (Step \ref{flow:2}).


\begin{algorithm}[t!]
\label{alg:Evaluate}
\caption{Evaluate Step}
\KwIn{Instantiated network $\delta$, Number of simulations $\bf{NSim}$}
\KwOut{Ranked Ordering of actions $\alpha^{\delta}$}
\begin{footnotesize}
$tree = Initialize\_K\_Level\_Tree()$;\\\label{evalflow:0}
$counter = 0$;\\\label{evalflow:1}
\While {$counter++ < \bf{NSim}$} {
	$K\_Node\_Act =  FindStep(tree)$;\\\label{evalflow:2}
	$LT\_Reward = SimulateStep(K\_Node\_Act)$;\\\label{evalflow:3}
	$UpdateStep(tree, LT\_Reward, K\_Node\_Act)$;\\\label{evalflow:4}
	}
	$\alpha^{\delta} = Get\_All\_Leaf\_Values(tree)$;\\\label{evalflow:7}
	return $\alpha^{\delta}$;
\end{footnotesize}
\end{algorithm}
\setlength{\textfloatsep}{5pt}

\textbf{Evaluate Step} Algorithm 2 generates the $\alpha^{\delta}$ list for a single instantiated network $\delta \in [1,\Delta]$. This algorithm works similarly for all instantiated networks. For each instantiated network, the Evaluate Step uses $\bf{NSim}$ (we use $2^{10}$) number of MC simulations to evaluate the long term reward achieved by taking actions in that network. Due to the combinatorial action space, the Evaluate Step uses a UCT \cite{kocsis2006bandit} driven approach to strategically choose the actions whose long term rewards should be calculated. UCT has been used to solve POMDPs in \cite{silver2010monte, yadav2015preventing}, but these algorithms suffer from a ${N \choose K}$  branching factor (where $K$ is number of nodes picked per round, $N$ is number of network nodes). We exploit the structure of our domain by creating a $K$-level UCT tree which has a branching factor of just $N$ (explained below). This $K$-level tree allows storing reward values for smaller sized node subsets as well (instead of just $K$ sized subsets), which helps in guiding the UCT search better.


Algorithm 2 takes an \textit{instantiated} network and creates the aforementioned $K$-level tree for that network. The first level of the tree has $N$ branches (one for each network node). For each branch $i$ in the first level, there are $N-1$ branches in the second tree level (one for each network node, except for node $i$, which was covered in the first level). Similarly, for every branch $j$ in the $m^{th}$ level ($m \in [2,K-1]$), there are $N-m$ branches in the $(m+1)^{th}$ level. Theoretically, this tree grows exponentially with $K$, however, the values of $K$ are usually small in practice (e.g., 4). 


In this $K$ level tree, each leaf node represents a particular POMDP action of $K$ network nodes. Similarly, every non-leaf tree node $v$ represents a subset $S_v$ of network nodes. Each tree node $v$ maintains a value $R_v$, which represents the average long term reward achieved by taking our POMDP's actions (of size $K$) which contain $S_v$ as a subset. For example, in Figure \ref{fig:uncertainG}, if $K=5$, and for tree node $v$, $S_v = \{A,B,C,D\}$, then $R_v$ represents the average long term reward achieved by taking POMDP actions $A_1 = \{A,B,C,D,E\}$ and $A_2 = \{A,B,C,D,F\}$, since both $A_1$ and $A_2$ contain $S_v = \{A,B,C,D\}$ as a subset. To begin with, all nodes $v$ in the tree are initialized with $R_v=0$ (Step \ref{evalflow:0}). By running $\bf{NSim}$ number of MC simulations, we generate good estimates of $R_v$ values for each tree node $v$.

Each node in this $K$-level tree runs a UCB1 \cite{kocsis2006bandit} implementation of a multi-armed bandit. The arms of the multi-armed bandit running at tree node $v$ correspond to the child branches of node $v$ in the $K$-level tree. Recall that each child branch corresponds to a network node. The overall goal of all the multi-armed bandits running in the tree is to construct a POMDP action of size $K$ (by traversing a path from the root to a leaf), whose reward is then calculated in that MC simulation (explained in Algorithm 3). Every MC simulation consists of three steps: Find Step (Step \ref{evalflow:2}), Simulate Step (Step \ref{evalflow:3}) and Update Step (Step \ref{evalflow:4}).

\begin{algorithm}[t!]
\label{alg:FindStep}
\caption{FindStep}
\KwIn{$K$ level deep tree - $tree$}
\KwOut{Action set of size $K$ nodes - $Act$}
\begin{footnotesize}
$Act = \Phi$;\\\label{findflow:1}
$tree\_node  = tree.Root$;\\\label{findflow:2}
\While {$is\_Leaf(tree\_node) == false$} {
	$MAB_{node} = Get\_UCB\_at\_Node(node)$;\\\label{findflow:3}
	$next\_node = Ask\_UCB(MAB_{node})$;\\\label{findflow:4}
	$Act = Act \cup next\_node$;\\
	$tree\_node = tree\_node.branch(next\_node)$;\\
	}
	return $Act$;\\
\end{footnotesize}
\end{algorithm}
\setlength{\textfloatsep}{5pt}

\textbf{Find Step}: The Find Step takes a $K$-level tree for an instantiated network and  \textit{finds} a $K$ node action, which is used in the Simulate Step. Algorithm 3 details the process of \textit{finding} this $K$ node action, which is found by traversing a path from the root node to a leaf node, one edge/arm at a time. Initially, we begin at the root node with an empty action set of size 0 (Steps \ref{findflow:1} and \ref{findflow:2}). For each node that we visit on our way from the root to a leaf, we use its multi-armed bandit (denoted by $MAB_{node}$ in Step \ref{findflow:3}) to choose which tree node do we visit next (or, which network node do we add to our action set). We get a $K$ node action upon reaching a leaf.

\textbf{Simulate Step}: The Simulate Step takes a $K$ node action from the Find Step, to \textit{evaluate} the long term reward of taking that action (called $Act$) in the instantiated network. Assuming that $T_0$ interventions remain (i.e., we have already conducted $T - T_0$ interventions), the Simulate Step first uses  action $Act$ in the generative model $\Lambda$ to generate a reward $r_0$. For all remaining $(T_0 - 1)$ interventions, Simulate Step uses a rollout policy to randomly select K node actions, which are then used in the generative model $\Lambda$ to generate future rewards $r_i \mbox{ }\forall \mbox{ } i \in [1, T_0 - 1]$ . Finally, the long term reward returned by Simulate Step is $r_0 + r_1 + ... + r_{T_0-1}$.


\textbf{Update Step}: The Update Step uses the long term reward returned by Simulate Step to update relevant $R_v$ values in the $K$-level tree. It updates the $R_v$ values of all nodes $v$ that were traversed in order to find the $K$ node action in the Find Step. First, we get the tree's leaf node corresponding to the $K$ node action that was returned by the Find Step. Then, we go and update $R_v$ values for all ancestors (including the root) of that leaf node. 

After running the Find, Simulate and Evaluate for $\bf{NSim}$ simulations, we return the $R_v$ values of all leaf nodes as the $\alpha^{\delta}$ list. Recall that we then find the expected reward $r_i$ of taking the $i^{th}$ action, by taking an expectation of rewards across the $\alpha^{\delta}$ lists. Finally, TASP returns the action $\kappa = \argmax_j r_j$.

\subsubsection{Top layer: Using Graph Partitioning}\label{sec:parti} We now explain HEAL's top layer, in which we use METIS \cite{lasalle2013multi}, a state-of-the-art graph partitioning technique, to subdivide our original uncertain network into different partitioned networks. These partitioned networks form the ensemble of \textit{intermediate POMDPs} (in Figure \ref{fig:Flow}) in HEAL. Then, TASP is invoked on each intermediate POMDP independently, and their solutions are aggregated to get the final DIME solution. We try two different partitioning/aggregation techniques, which leads to two variants of HEAL:

\textbf{K Partition Variant (HEAL): } Given the \textit{uncertain} network $G$ and the parameters $K$, $L$ and $T$ as input, we first partition the uncertain network into $K$ partitions. In each round from 1 to $T$, we invoke the bottom layer TASP algorithm to select 1 node from each of the $K$ clusters. These singly selected nodes from the $K$ clusters give us an action of $K$ nodes, which is given to shelter officials to execute. Based on the \textit{observation} (about uncertain edges) that officials get while executing the action, we update the partition networks (which are input to the \textit{intermediate POMDPs}) by either replacing the \textit{observed} uncertain edges with certain edges (if the edge was \textit{observed} to exist in reality) or removing the uncertain edge altogether (if the edge was \textit{observed} to \textit{not exist} in reality). The list of $K$ node actions that Algorithm 4 generates serves as an online policy for use by the homeless shelter. 

\textbf{T Partition Variant (HEAL-T): } Given the \textit{uncertain} network $G$ and the parameters $K$, $L$ and $T$ as input, we first partition the uncertain network into $T$ partitions and TASP picks $K$ nodes from the $i^{th}$ partition ($i \in [1,T]$) in the $i^{th}$ round. 

%

\section{Experimental Results}\label{sec:exp}
In this section, we analyze HEAL and HEAL-T's performance in a variety of settings. All our experiments are run on a 2.33 GHz 12-core Intel machine having 48 GB of RAM. All experiments are averaged over 100 runs. We use a metric of ``\textit{Indirect Influence}" throughout this section, which is number of nodes ``\textit{indirectly}" influenced by intervention participants. For example, on a 30 node network, by selecting 2 nodes each for 10 interventions (horizon), 20 nodes (a lower bound for any strategy) are influenced with certainty. However, the total number of influenced nodes might be 26 (say) and thus, the \textit{Indirect Influence} is $26-20 = 6$. In all experiments, the propagation and existence probability values on all network edges were uniformly set to $0.1$ and $0.6$, respectively. This was done based on findings in Kelly et. al.\cite{kelly1997randomised}. However, we relax these parameter settings later in the section. All experiments are statistically significant under bootstrap-t ($\alpha = 0.05$).

\textbf{Baselines: } We use two algorithms as baselines. We use PSINET-W as a benchmark as it is the most relevant previous algorithm, which was shown to outperform heuristics used in practice; however, we also need a point of comparison when PSINET-W does not scale. No previous algorithm in the influence maximization literature accounts for uncertain edges and uncertain network state in solving the problem of sequential selection of nodes; in-fact we show that even the standard Greedy algorithm \cite{kempe2003maximizing,golovin2011adaptive} has no approximation guarantees as our problem is not adaptive submodular. Thus, we modify Greedy by replacing our uncertain network with a certain network (in which each uncertain edge $e$ is replaced with a certain edge $e_0$ having propagation probability $p(e_0) = p(e) \times u(e)$), and then run the Greedy algorithm on this \textit{certain network}. We use the Greedy algorithm as a baseline as it is the best known algorithm known for influence maximization and has been analyzed in many previous papers \cite{cohen2014sketch,Borgs14,tang2014influence,kempe2003maximizing,leskovec2007cost,golovin2011adaptive}. 


\textbf{Datasets: } We use \textit{four real world social networks} of homeless youth, provided to us by our collaborators. All four networks are friendship based social networks of homeless youth living in different areas of a big city in USA (name withheld for anonymity). The first and second networks are of homeless youth living in two large areas (denoted by VE and HD to preserve anonymity), respectively. These two networks (each having $\sim$150-170 nodes, 400-450 edges) were created through surveys and interviews of homeless youth (conducted by our collaborators) living in these areas. The third and fourth networks are relatively small-sized online social networks of these youth created from their Facebook (34 nodes, 120 edges) and MySpace (107 nodes, 803 edges) contact lists, respectively. When HEALER is deployed, we anticipate even larger networks, (e.g., 250-300 nodes) than the ones we have in hand and we also show run-time results on artificial networks of these sizes.

\begin{figure}[t]
\vspace{-2pt}
\subfloat[\small Solution Quality]{\includegraphics[height=0.9in,width=0.5\columnwidth]{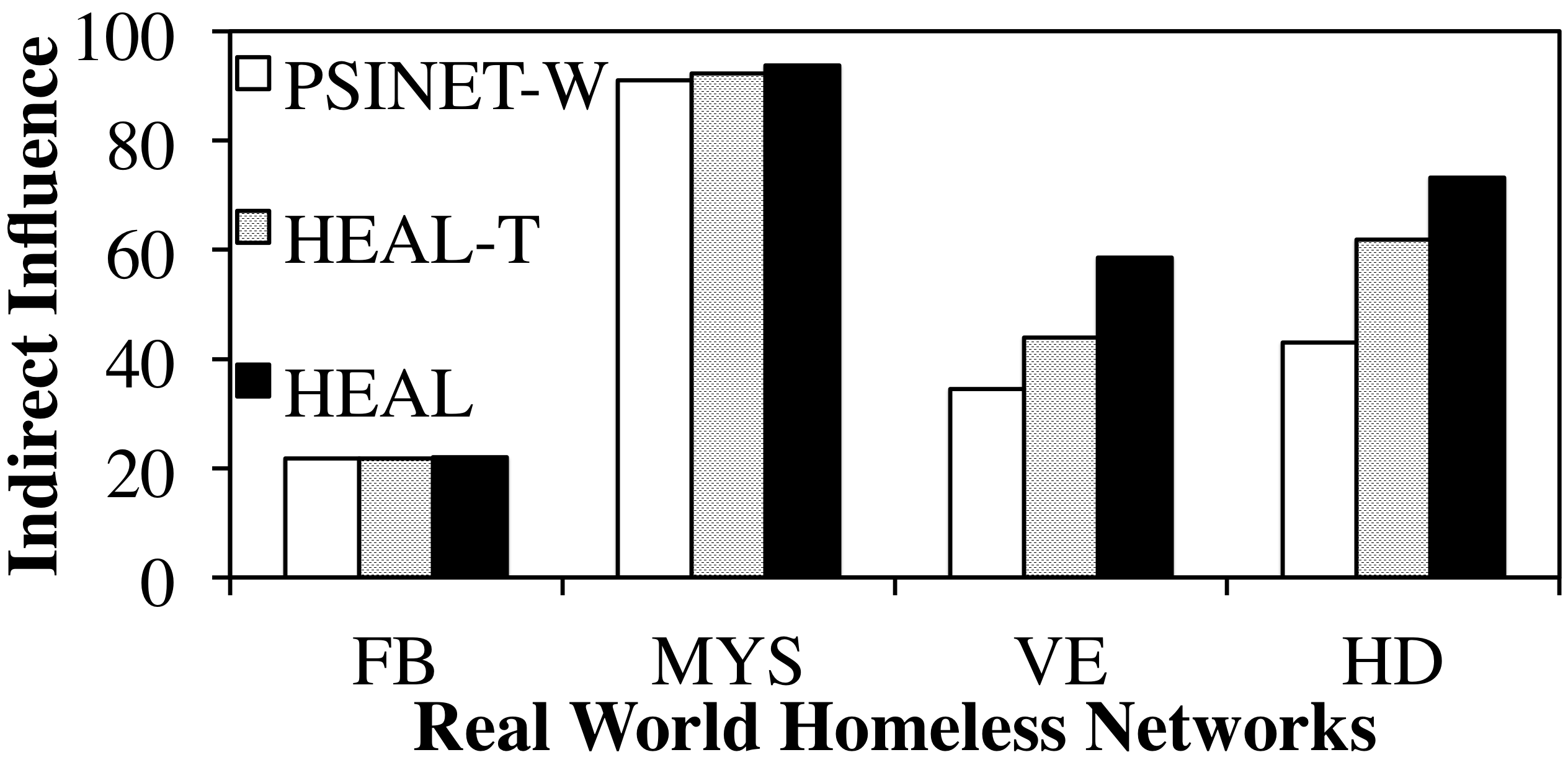}\label{fig:SolQual}}
\subfloat[\small Runtime]{\includegraphics[height=0.9in,width=0.5\columnwidth]{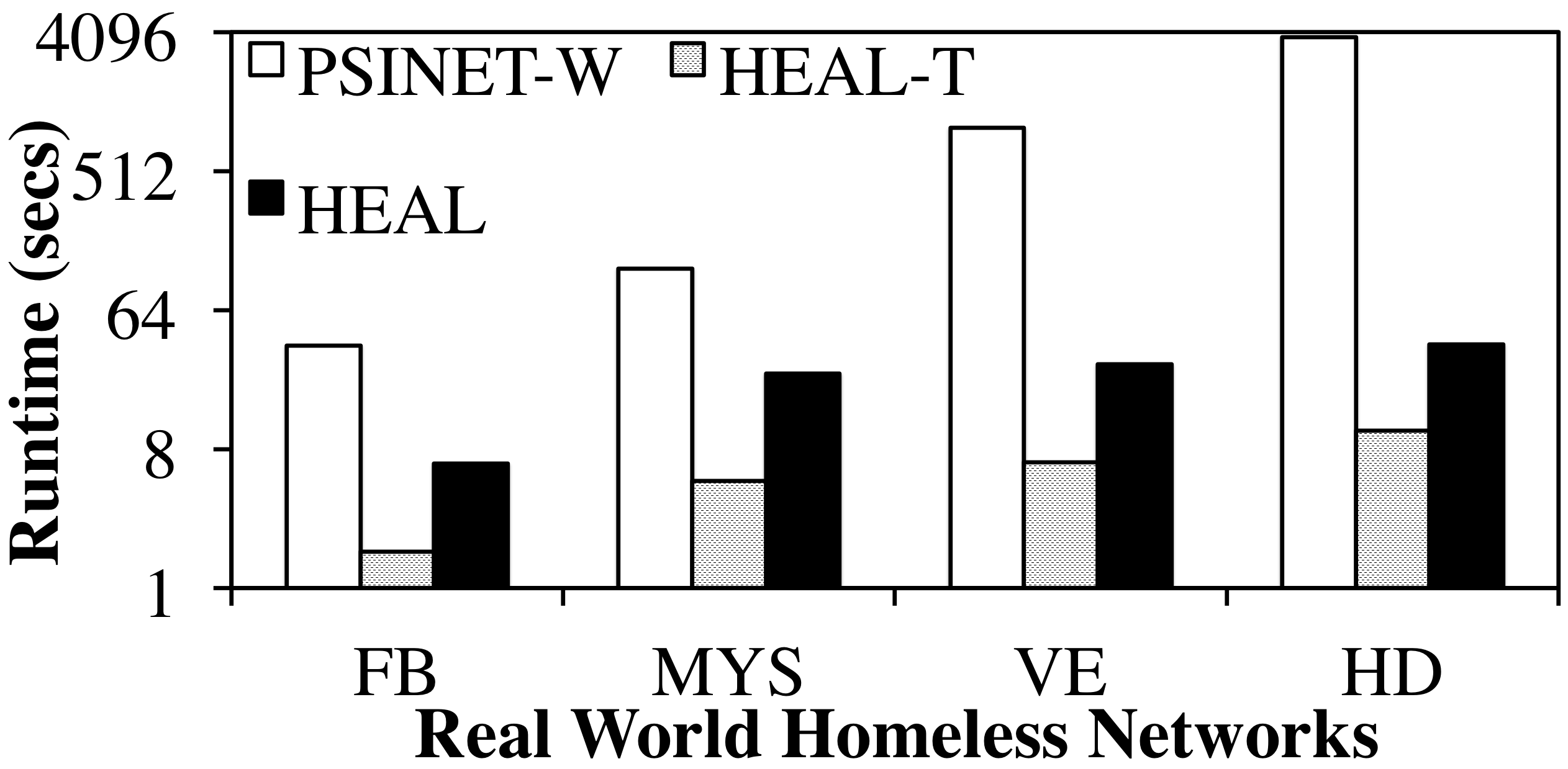}\label{fig:Runtime}}
\caption{Solution Quality and Runtime on Real World Networks}
\end{figure}  

\textbf{Solution Quality/Runtime Comparison. }We compare \textit{Indirect Influence} and run-times of HEAL, HEAL-T and PSINET-W on all four real-world networks. We set $T=5$ and $K=2$ (since PSINET-W fails to scale up beyond $K=2$ as shown later). Figure \ref{fig:SolQual} shows the \textit{Indirect Influence} of the different algorithms on the four networks. The X-axis shows the four networks and the Y-axis shows the \textit{Indirect Influence} achieved by the different algorithms. This figure shows that (i) HEAL outperforms all other algorithms on every network; (ii) \textit{it achieves $\sim$70\% improvement over PSINET-W} in VE and HD networks; (iii) it achieves $\sim$25\% improvement over HEAL-T. The difference between HEAL and other algorithms is not significant in the Facebook (FB) and MySpace (MYS) networks, as HEAL is already influencing almost all nodes in these two relatively small networks. Thus, in experiments to come, we focus more on the VE and HD networks. 


Figure \ref{fig:Runtime} shows the run-time of all algorithms on the four networks. The X-axis shows the four networks and the Y-axis (in log scale) shows the run-time (in seconds). This figure shows that (i) \textit{HEAL achieves a 100X speed-up over PSINET-W}; (ii) PSINET-W's run-time increases exponentially with increasing network sizes; (iii) HEAL runs 3X slower than HEAL-T but achieves 25\% more \textit{Indirect Influence}. Hence, HEAL is our algorithm of choice.


Next, we check if PSINET-W's run-times become worse on larger networks. Because of lack of larger real-world datasets, we create relatively large artificial Watts-Strogatz networks (model parameters $p=0.1,k=7$). Figure \ref{fig:big-runtime} shows the run-time of all algorithms on Watts-Strogatz networks. The X-axis shows the size of networks and the Y-axis (in log scale) shows the run-time (in seconds). This figure shows that \textit{PSINET-W fails to scale beyond 180 nodes, whereas HEAL runs within 5 minutes}. Thus, PSINET-W fails to scale-up to network sizes that are of importance to us. 

\begin{figure}[t]
\vspace{-2pt}
\subfloat[\small VE Network]{\includegraphics[height=0.9in,width=0.5\columnwidth]{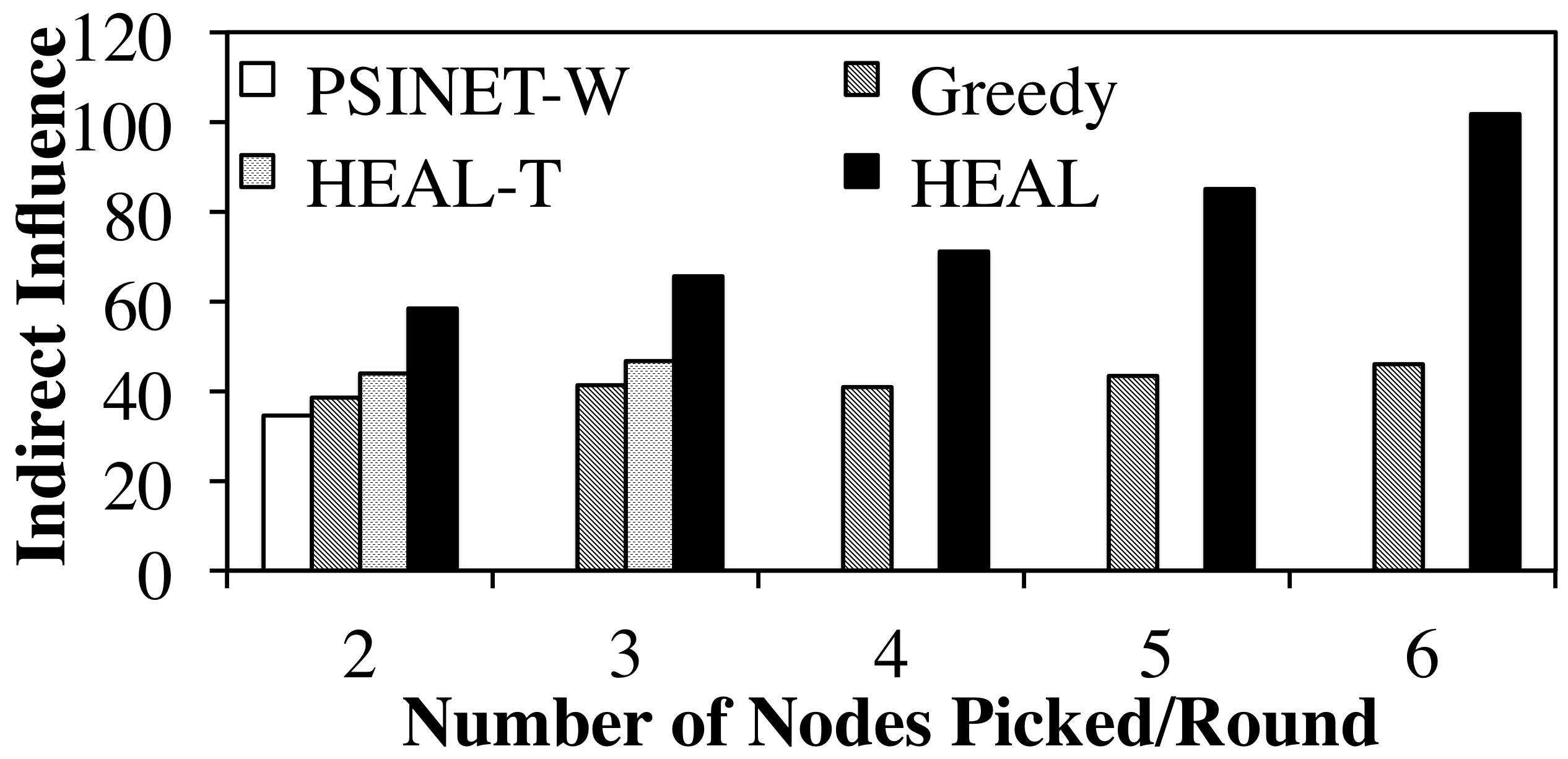}\label{fig:VeniceNet}}
\subfloat[\small HD Network]{\includegraphics[height=0.9in,width=0.5\columnwidth]{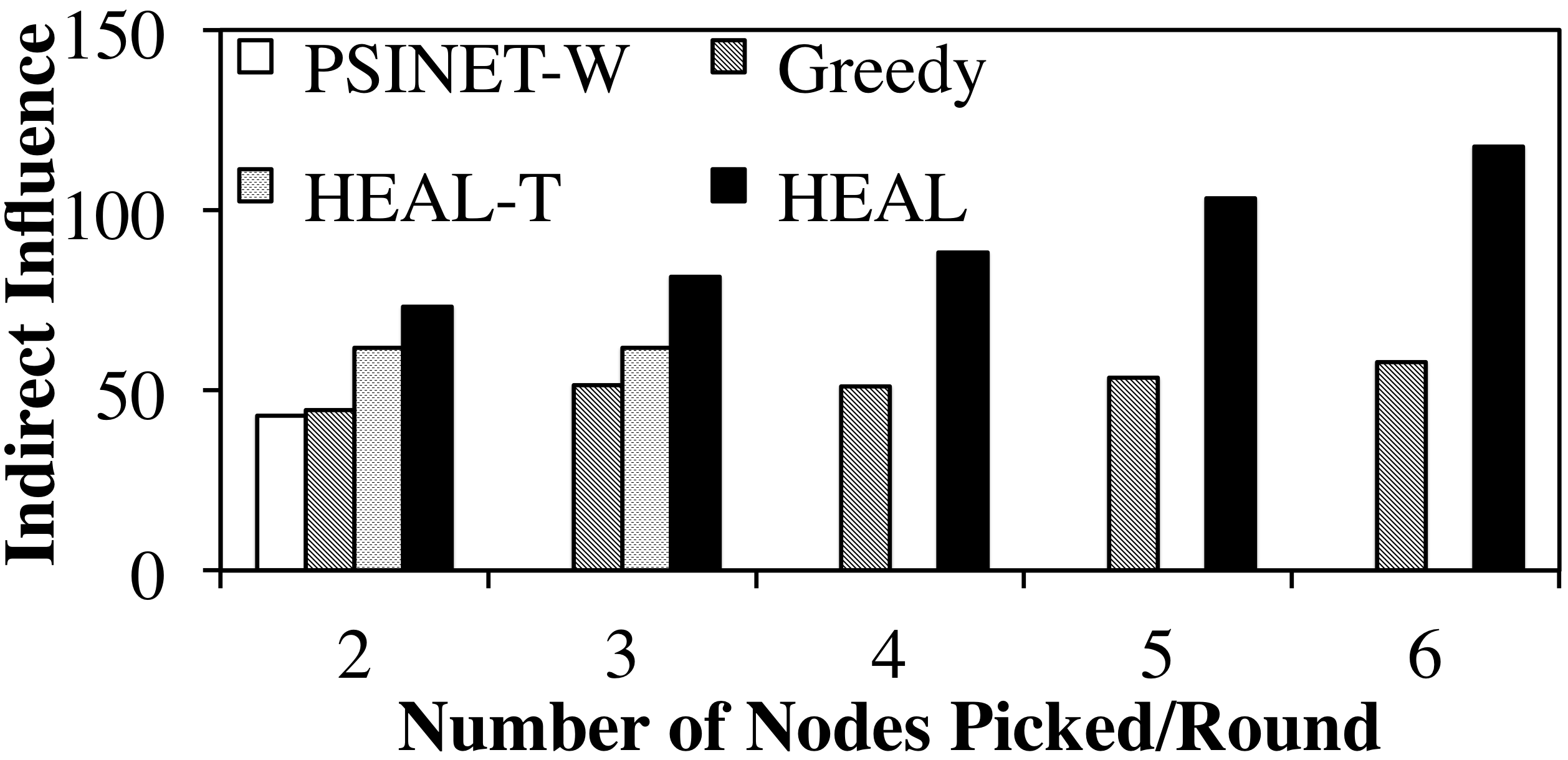}\label{fig:HollywoodNet}}
\caption{Scale up in number of nodes picked per round}
\end{figure}

\textbf{Scale Up Results. } Not only does PSINET-W fail in scaling up to larger network sizes, it even fails to scale-up with increasing number of nodes picked per round (or $K$), on our real-world networks. Figures \ref{fig:VeniceNet} and \ref{fig:HollywoodNet} show the \textit{Indirect Influence} achieved by HEAL, HEAL-T, Greedy and PSINET-W on the VE and HD networks respectively ($T=5$), as we scale up $K$ values. The X-axis shows increasing $K$ values, and the Y-axis shows the \textit{Indirect Influence}. These figures show that (i) PSINET-W and HEAL-T fail to scale up  -- they cannot handle more than $K=2$ and $K=3$ respectively (thereby not fulfilling real world demands); (ii) HEAL outperforms all other algorithms, and the difference between HEAL and Greedy increases linearly with increasing $K$ values. Also, \textit{in the case of $K=6$, HEAL runs in less than $40.12$ seconds on the HD network and $34.4$ seconds on the VE network}.


Thus, Figures \ref{fig:SolQual}, \ref{fig:Runtime}, \ref{fig:VeniceNet} and \ref{fig:HollywoodNet} show that PSINET-W (the best performing algorithm from previous work) fails to scale up with increasing network nodes, and with increasing $K$ values. Even for $K=2$ and moderate sized networks, it runs very slowly. Moreover, HEAL is the best performing algorithm that runs quickly, provides high-quality solutions, and can scale-up to real-world demands. Since only HEAL and Greedy scale up to $K=6$, we now analyze their performance in detail.

\begin{figure}[t]
\subfloat[\small Solution Quality]{\includegraphics[height=0.9in,width=0.5\columnwidth]{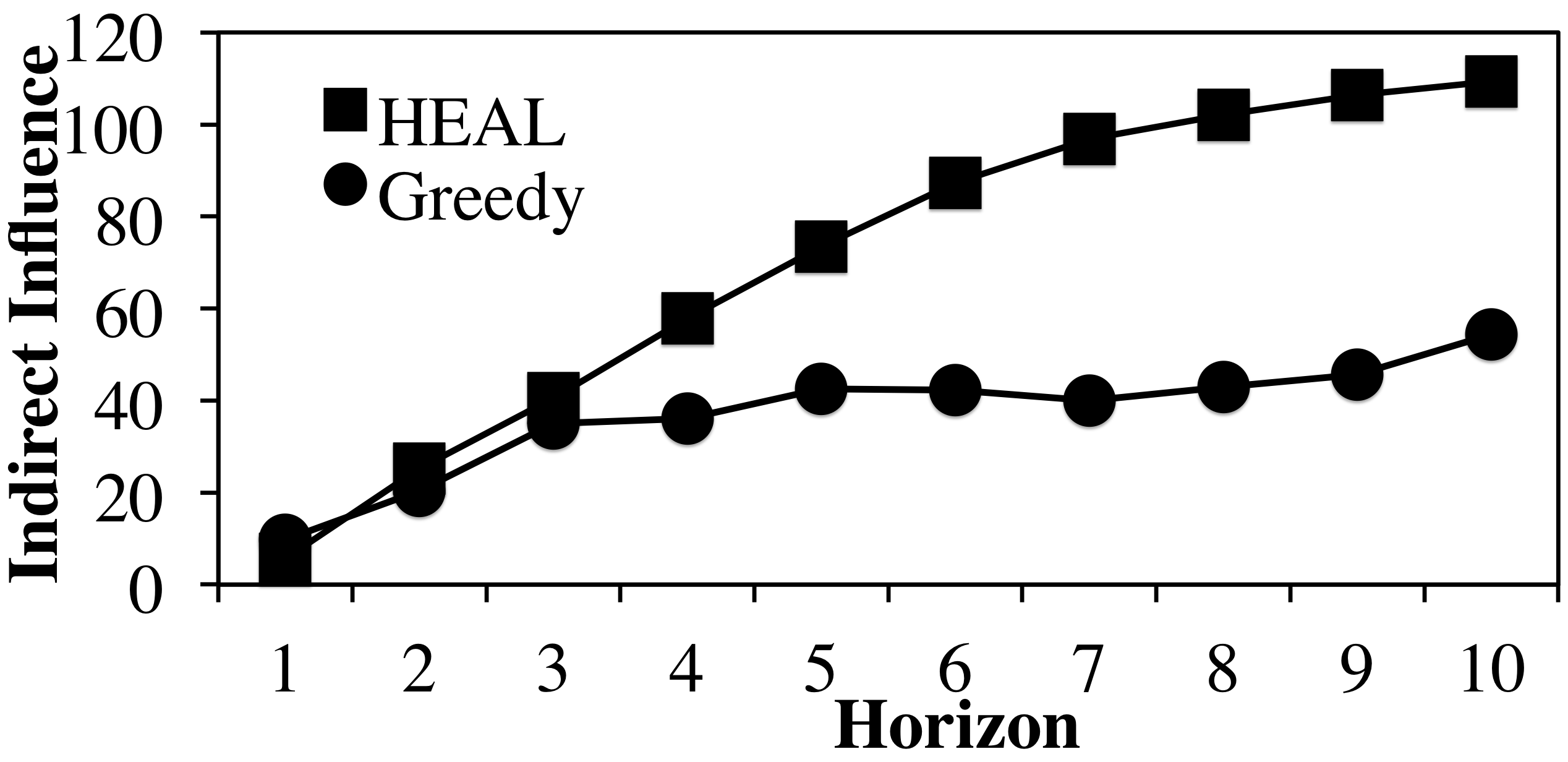}\label{fig:Horizon-HollySolQual}}
\subfloat[\small Maximum Relative Gain]{\includegraphics[height=0.9in,width=0.5\columnwidth]{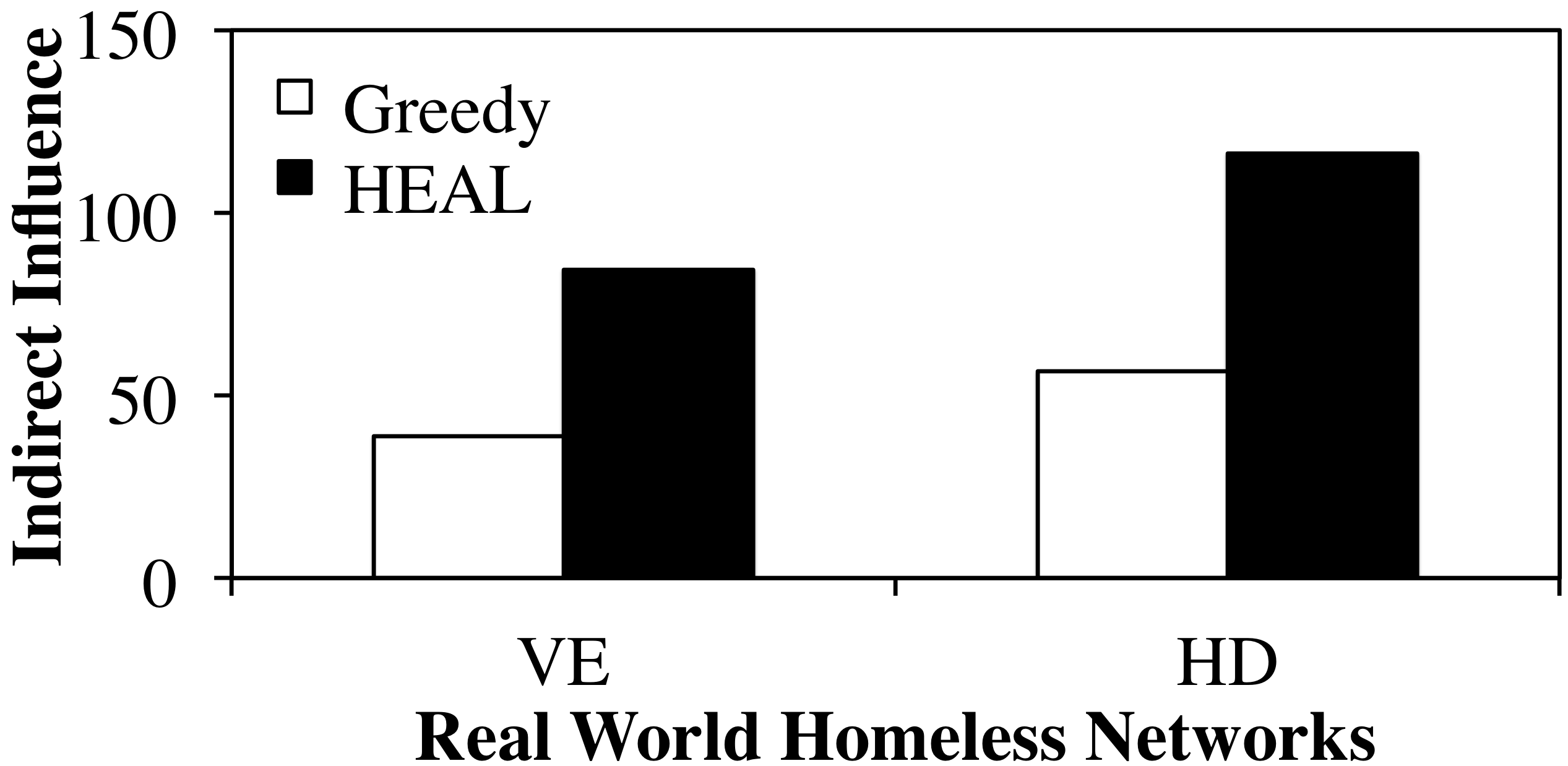}\label{fig:ws-SolQual}}
\caption{Horizon Scale up \& Maximum Gain on HD Network}
\end{figure}

\textbf{Scaling up Horizon. }Figure \ref{fig:Horizon-HollySolQual} shows HEAL and Greedy's \textit{Indirect Influence} in the HD network, with varying horizons (see appendix for VE network results). The X-axis shows increasing horizon values and the Y-axis shows the \textit{Indirect Influence} ($K=2$). This figure shows that the relative difference between HEAL and Greedy increases significantly with increasing $T$ values.

Next, we scale up $K$ values with increased horizon settings to find the maximum attainable solution quality difference between HEAL and Greedy. Figure \ref{fig:ws-SolQual} shows the \textit{Indirect Influence} achieved by HEAL and Greedy (with $K=4$ and $T=10$) on the VE and HD networks. The X-axis shows the two networks and the Y-axis shows the \textit{Indirect Influence}. This figure shows that with these settings, \textit{HEAL achieves $\sim$110\% more Indirect Influence than Greedy (i.e., more than a 2-fold improvement) in the two real-world networks}.


\begin{figure}[t]
\CenterFloatBoxes
\resizebox{\columnwidth}{!}{
\subfloat[\small Deviation Tolerance]{\raisebox{-.43\height}{\includegraphics[height=0.8in,width=0.5\columnwidth]{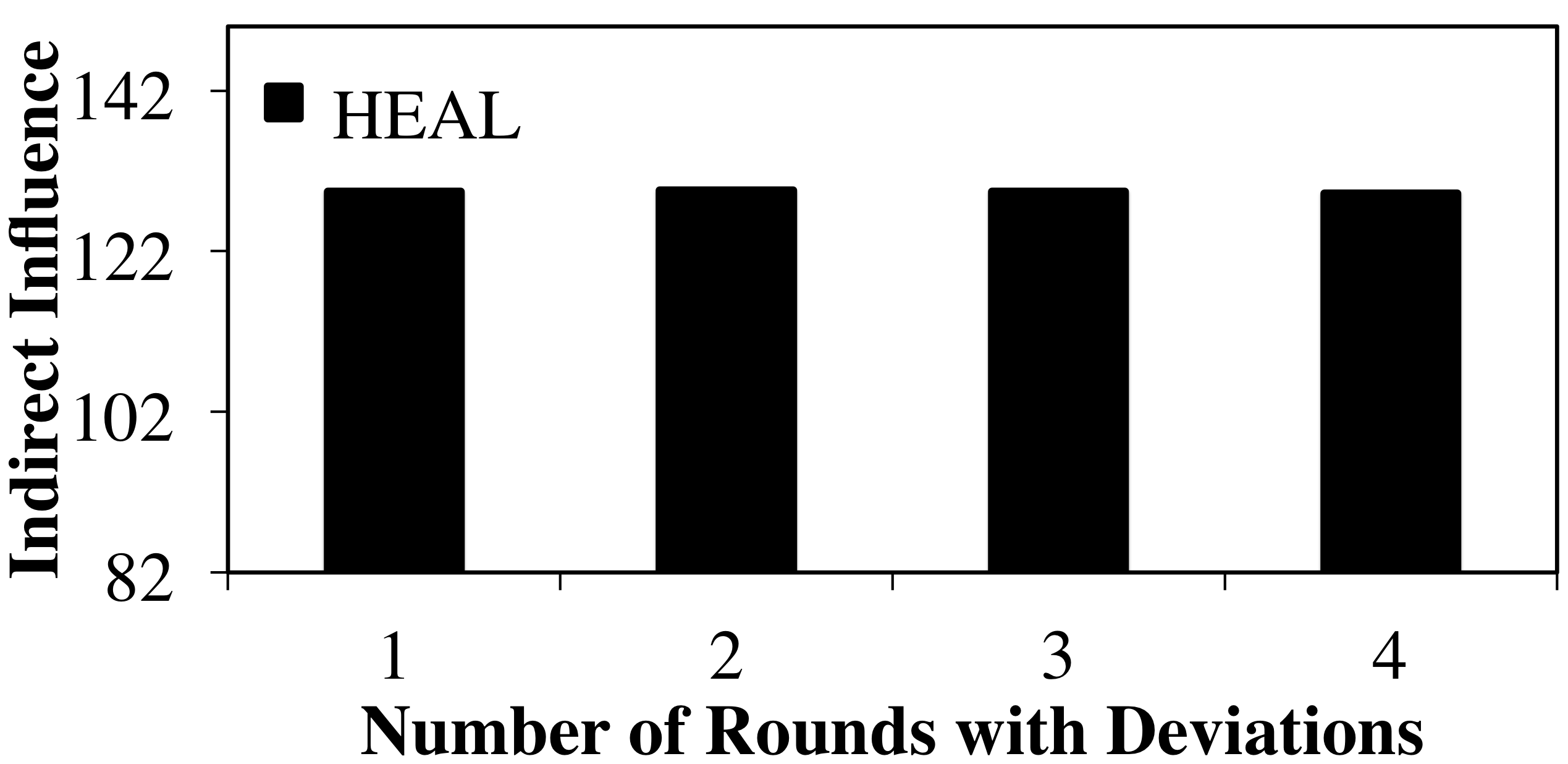}\label{fig:ws-deviation}}}
\subfloat[\small \% Increase over Greedy]{
\resizebox{1.5in}{0.4in}{
\begin{tabular}{|l|c|c|c|}\hline
\theadfont\diagbox[width=4em]{$p(e)$}{$u(e)$}&
\thead{$0.1$}&\thead{$0.2$}&\thead{$0.3$}\\    \hline
$0.7$ & $45.62$ & $44.37$ & $30.85$ \\    \hline
$0.6$ & $48.95$ & $24.56$ & $30$ \\    \hline
$0.5$ & $29.5$ & $55.18$ & $28.21$ \\    \hline
\end{tabular}\label{tab:inc-over-greedy}}}
}
\caption{Deviation Tolerance \& HEAL vs Greedy}
\end{figure}


\textbf{HEAL vs Greedy. }Figure \ref{tab:inc-over-greedy} shows the percentage increase (in \textit{Indirect Influence}) achieved by HEAL over Greedy with varying $u(e)$/$p(e)$ values. The columns and rows of Figure \ref{tab:inc-over-greedy} show varying $u(e)$ and $p(e)$ values respectively. The values inside the table cells show the percentage increase (in \textit{Indirect Influence}) achieved by HEAL over Greedy when both algorithms plan using the same $u(e)$/$p(e)$ values. For example, with $p(e)=0.7$ and $u(e)=0.1$, HEAL achieves 45.62\% more \textit{Indirect Influence} than Greedy. This figure shows that \textit{HEAL continues to outperform Greedy across varying $u(e)$/$p(e)$ values}. Thus, on a variety of settings, HEAL dominates Greedy in terms of both \textit{Indirect Influence} and run-time.

\textbf{Deviation Tolerance. } We show HEAL's tolerance to deviation by replacing a fixed number of actions recommended by HEAL with randomly selected actions. Figure \ref{fig:ws-deviation} shows the variation in \textit{Indirect Influence} achieved by HEAL ($K=4$,$T=10$) with increasing number of random deviations from the recommended actions. The X-axis shows increasing number of deviations and the Y-axis shows the \textit{Indirect Influence}. For example, when there were 2 random deviations (i.e., two recommended actions were replaced with random actions), HEAL achieves 100.23 \textit{Indirect Influence}. This figure shows that HEAL is highly deviation-tolerant.


\begin{figure}[t]
\CenterFloatBoxes
\resizebox{\columnwidth}{!}{
\subfloat[\small Artificial Networks]{\raisebox{-.43\height}{\includegraphics[height=0.8in,width=0.5\columnwidth]{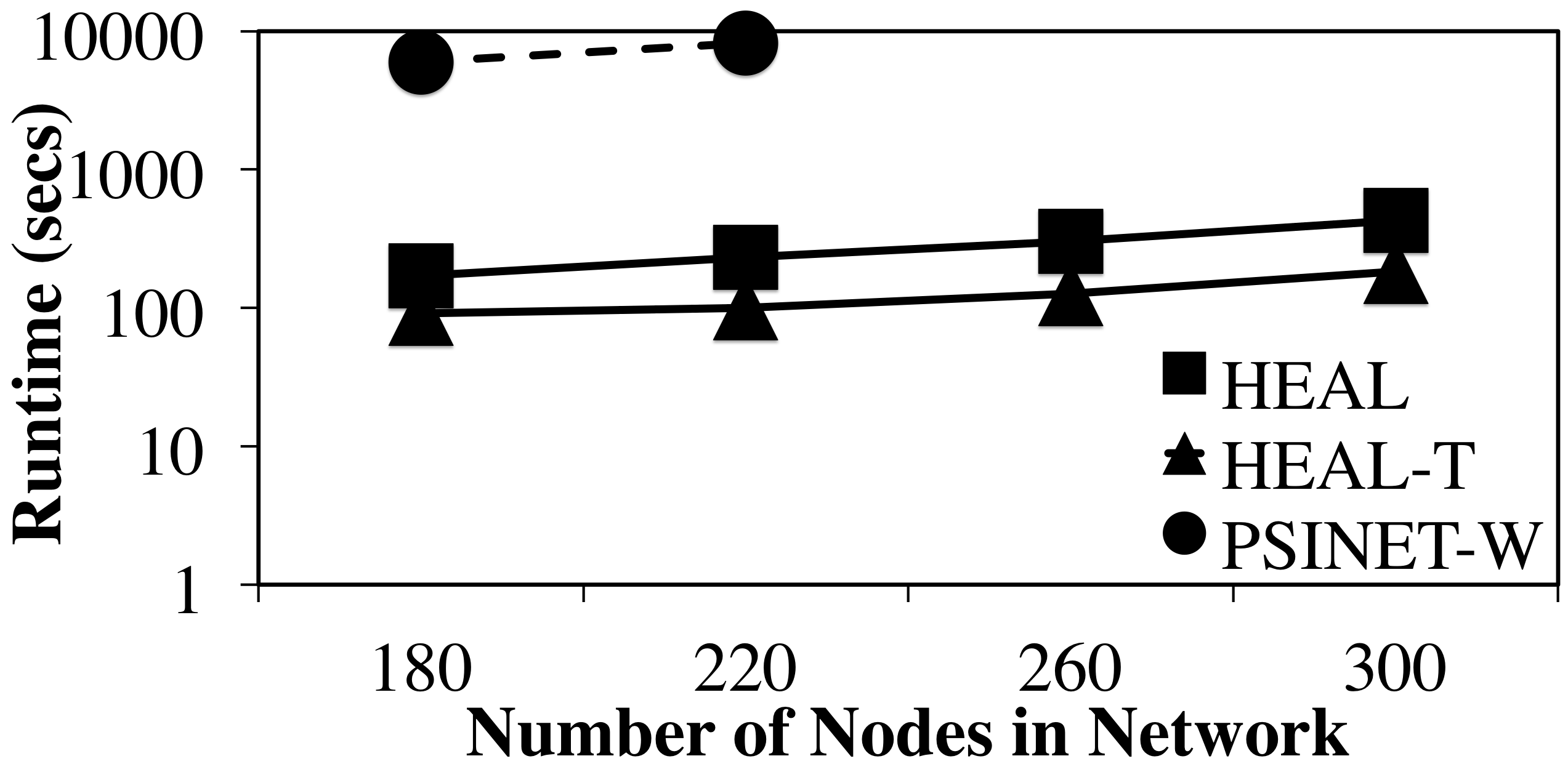}\label{fig:big-runtime}}}
\subfloat[\small \% Loss in HEAL Solution]{
\resizebox{1.5in}{0.4in}{
\begin{tabular}{|l|c|c|c|}\hline
\theadfont\diagbox[width=4em]{$p(e)$}{$u(e)$}&
\thead{$0.1$}&\thead{$0.2$}&\thead{$0.3$}\\    \hline
$0.7$ & $24.42$ & $21.02$ & $16.85$ \\    \hline
$0.6$ & $0.0$ & $18.26$ & $12.46$ \\    \hline
$0.5$ & $11.58$ & $10.53$ & $8.11$ \\    \hline
\end{tabular}\label{tab:sensitive}}}
}
\caption{Artificial Networks And Sensitivity Analysis}
\end{figure}



\textbf{Sensitivity Analysis. } Finally, we test the robustness of HEAL's solutions in the HD network (see appendix for VE network results), by allowing for error in HEAL's understanding of $u(e)$/$p(e)$ values. We consider the case that $u(e)=0.1$ and $p(e)=0.6$ values that HEAL plans on, are wrong. Thus, HEAL plans its solutions using $u(e)=0.1$ and $p(e)=0.6$, but those solutions are evaluated on different (correct) $u(e)$/$p(e)$ values to get \textit{estimated solutions}. These \textit{estimated solutions} are compared to \textit{true solutions} achieved by HEAL if it planned on the correct $u(e)$/$p(e)$ values. Figure \ref{tab:sensitive} shows the percentage difference (in \textit{Indirect Influence}) between the \textit{true} and \textit{estimated} solutions, with varying $u(e)$ and $p(e)$ values. For example, when HEAL plans its solutions with wrong $u(e)=0.1$/$p(e)=0.6$ values (instead of correct $u(e)=0.3$/$p(e)=0.5$ values), it suffers a 8.11\% loss. This figure shows that HEAL is relatively robust to errors in its understanding of $u(e)$/$p(e)$ values, as it only suffers an average-case loss of $\sim$ 15\%.


\section{Conclusion}
This paper focuses on the important problem of selecting participants of sequentially deployed interventions, which are organized by homeless shelters to spread awareness about HIV prevention practices among homeless youth. This is an extremely important problem as homeless youth are at high-risk to HIV ($\sim$10\% of homeless youth are HIV positive). While previous work tries to solve this problem, it \textit{simply fails to scale up to real world sizes and demands}. It runs out of memory on large networks, with increased number of intervention participants, and runs very slowly on moderate sized networks. In this paper, we develop HEALER, a new software agent for solving this problem which scales up to real world demands. HEALER casts the problem as a POMDP and uses a completely novel suite of algorithms (HEAL, TASP and Evaluate) to achieve a 100X speedup over state-of-the-art algorithms while outperforming them by 70\% in terms of solution quality. More than that, it runs when previous algorithms can't scale up. Also, HEALER saves homeless shelters' thousands of dollars and many months of time by generating uncertain networks at low cost using its Facebook application. Finally, we show some novel theoretical hardness results about the problem that HEALER solves. HEALER is fully ready to be deployed in the real world, in collaboration with a homeless shelter. The shelter officials have tested HEALER's components and their feedback has been positive. HEALER's deployment is expected to commence in early Spring 2016. This is an extended version of our AAMAS 2016 paper by the same name. For the conference version, please refer to \cite{yadav2016dime}.

\bibliography{writeup}
\bibliographystyle{abbrv}
\end{document}